%% file: arxiv.tex
\documentclass[letterpaper, 11pt]{article}

\usepackage[utf8]{inputenc}
\usepackage{geometry}
\usepackage[utf8]{inputenc}
\usepackage[colorlinks=true,linkcolor=blue,urlcolor=black,bookmarksopen=true, bookmarksopenlevel=1]{hyperref}
\usepackage{bookmark}
\usepackage{url}            
\usepackage{booktabs}       
\usepackage{amsfonts,enumitem}       
\usepackage{nicefrac}       
\usepackage{microtype}      

\usepackage{algorithmic}
\usepackage[linesnumbered,ruled,vlined, algo2e, resetcount]{algorithm2e}
\input{math_commands.tex}

\usepackage{graphicx}
\usepackage{subfigure}
\usepackage{booktabs} 
\usepackage{amsmath}
\usepackage{amsthm}
\usepackage{float}
\usepackage{bm}
\usepackage{amssymb}
\usepackage{bbm}
\usepackage{color}
\usepackage{mathtools}
\usepackage[super]{nth}
\usepackage{etoolbox}
\usepackage{adjustbox}
\usepackage{enumitem}

\DeclareRobustCommand{\sewoong}[1]{
{\begingroup\sethlcolor{green}\hl{(Sewoong:) #1}\endgroup}}

\def\reals{{\mathbb R}}
\def\prob{{\mathbb P}}

\def\cN{{\cal N}}
\def\eps{\varepsilon}
\def\E{\mathbb E}

\newtheorem{propo}{Proposition}[section]
\newtheorem{lemma}[propo]{Lemma}
\newtheorem{definition}[propo]{Definition}
\newtheorem{coro}[propo]{Corollary}
\newtheorem{thm}{Theorem}
\newtheorem{asmp}{Assumption}

\newtheorem{remark}[propo]{Remark}

\newcommand{\ip}[2]{\left\langle #1, #2 \right \rangle}
\def\reals{{\mathbb R}}
\def\prob{{\mathbb P}}
\usepackage{wrapfig}

\def\cN{{\cal N}}
\def\eps{\varepsilon}

\def\E{\mathbb E}

\def\HPTR{{\rm HPTR}}

\newcommand{\cD}{\mathcal D}

\newcommand{\cM}{\mathcal M}

\newcommand{\cP}{\mathcal P}

\title{Near Optimal Private and Robust Linear Regression}

\author{%
  Xiyang Liu\footnotemark[0] \thanks{Paul  Allen School of Computer Science \& Engineering, 
  University of Washington, 
  \texttt{xiyangl@cs.washington.edu}} 
  \and Prateek Jain\footnotemark[0] \thanks{Google Research, \texttt{prajain@google.com}}
 \and Weihao Kong\footnotemark[0] \thanks{Google Research, \texttt{weihaokong@google.com}} 
    \and
  Sewoong Oh\footnotemark[0] \thanks{Paul Allen School of Computer Science \& Engineering, 
  University of Washington, and  Google Research, 
  \texttt{sewoong@cs.washington.edu}} 
  \and
  Arun Sai Suggala\footnotemark[0] \thanks{Google Research, 
  \texttt{arunss@google.com}}
}

\date{}

\usepackage{natbib}
\usepackage{graphicx}

\begin{document}

\maketitle

\begin{abstract}
We study the canonical statistical estimation problem of linear regression from $n$ i.i.d.~examples under $(\varepsilon,\delta)$-differential privacy when some response variables are adversarially corrupted.   We propose a variant of the popular differentially private stochastic gradient descent (DP-SGD) algorithm with two innovations: a full-batch gradient descent to improve sample complexity and a novel adaptive clipping to guarantee robustness. When there is no adversarial corruption, this algorithm improves upon the existing state-of-the-art approach and achieves a near optimal sample complexity. Under label-corruption, this is the first efficient linear regression algorithm to guarantee both $(\varepsilon,\delta)$-DP and robustness. Synthetic experiments confirm the superiority of our approach.  
\end{abstract}

\section{Introduction} 
\label{sec:intro} 

Differential Privacy (DP) is a widely accepted notion of privacy introduced by \citet{dwork2006calibrating}, which is now standard in industry and government \citep{apple,google1,google2,census}.  
A query to a database is said to be 
$(\varepsilon,\delta)$-differentially private if a strong adversary who knows all other entries cannot identify with high confidence whether you participated in the database or not. The parameters $\varepsilon$ and $\delta$ restrict the  Type-I and Type-II errors achievable by the adversary in this hypothesis testing \citep{kairouz2015composition}. Smaller $\varepsilon>0$ and $\delta\in[0,1]$ imply stronger privacy guarantees. 

Although significant advances have been made recently in understanding the utility-privacy trade-offs in canonical statistical tasks, several important questions remain open.
We provide a survey in App.~\ref{app:related}. Consider a 
  canonical statistical task of linear regression  with $n$ i.i.d.~samples, $\{(x_i\in{\mathbb R}^d,y_i\in{\mathbb R})\}_{i=1}^n$, drawn from  $x_i\sim \cN(0,\Sigma)$, $y_i=x_i^\top w^*+z_i$, and $z_i\sim \cN(0,\sigma^2)$ for some true parameter $w \in{\mathbb R}^d$. The error is measured in $\|\hat w-w^*\|_\Sigma :=\|\Sigma^{1/2}(\hat w - w^*)\|$, which correctly accounts for the signal-to-noise ratio in each direction; in the direction of large eigenvalue of $\Sigma$, we have larger signal in $x_i$ but the noise $z_i$ remains the same; we expect smaller error in those directions, which is accounted for in $\|\hat w-w^*\|_\Sigma$.

When computational complexity is not concerned, 
the best known algorithm  is introduced by \citet{liu2022differential}, called High-dimensional Propose-Test-Release (HPTR). 
For linear regression, $n=  \tilde O(d/\alpha^2 + d/(\varepsilon\alpha))$ samples are sufficient for HPTR to achieve an error of $(1/\sigma)\|\hat w - w^*\|_\Sigma  = \alpha$ with high probability. 
 After a series of work surveyed in App.~\ref{app:related}, \citet{varshney2022nearly} achieve the best known sample complexity for an efficient algorithm: $n=\tilde O(\kappa^2d/\varepsilon + d/\alpha^2 + \kappa d / (\varepsilon\alpha))$. The last term  has an extra factor of $\kappa$, the condition number of the covariance $\Sigma$ of the covariates, and the first term is unnecessary.

 In this work, we propose a novel method (Algorithm~\ref{alg:main}) that uses full-batch gradient descent with adaptive clipping. Furthermore, using a intuitive but intricate analysis, we  improve this sample complexity. 

\begin{thm}[informal version of Thm.~\ref{thm:main} with no adversary]
\label{thm:main_informal} Alg.~\ref{alg:main} is $(\varepsilon,\delta)$-DP. Under the $(\Sigma,\sigma^2,w^*,K,a)$-model in Assumption~\ref{asmp:distribution}, 
 $n=\tilde O(d/\alpha^2 + \kappa^{1/2} d/ (\varepsilon \alpha))$ samples are sufficient for Alg.~\ref{alg:main} to achieve an error rate of $(1/\sigma)\|\hat w - w^*\|_\Sigma = \tilde O(\alpha)$, where $\kappa:=\lambda_{\rm max}(\Sigma)/\lambda_{\rm min}(\Sigma)$.
\end{thm}
That is, we can get rid of the first unnecessary term in alaysis of \cite{varshney2022nearly}, while also improving dependency on $\kappa$ term which is quite critical for practical applications. 

Perhaps surprisingly, we show that the same algorithm is also robust against label-corruption, where an adversary selects arbitrary $\alpha_{\rm corrupt}$ fraction of the data points and changes their response variables arbitrarily. 
When computational complexity is not concerned, 
the best known algorithm is again HPTR that also provides optimal robustness and $(\varepsilon,\delta)$-DP simultaneously, i.e.,  $n=  \tilde O(d/\alpha^2 + d/(\varepsilon\alpha))$ samples are sufficient for HPTR to achieve an error of $(1/\sigma)\|\hat w - w^*\|_\Sigma = \tilde O(\alpha )$ for any corruption bounded by $\alpha_{\rm corrupt}\leq  \alpha$. Note that this is a stronger adversary than the label-corruption we study in this paper; this adversary can corrupt both the covariate, $x_i$, and the response variable, $y_i$. Currently, there is no efficient algorithm that can guarantee both privacy and robustness for linear regression. Under a weaker adversary that only corrupts $y_i$'s, we provide the first efficient algorithm that achieves both privacy and robustness with a near-optimal sample complexity.

\begin{thm}[informal version of Thm.~\ref{thm:main} with adversarial label corruption]
\label{thm:main_informal_robust} Alg.~\ref{alg:main} is $(\varepsilon,\delta)$-DP. Under the hypotheses of Thm.~\ref{thm:main_informal} and under $\alpha_{\rm corrupt}$-corruption model of Assumption~\ref{asmp:corrupt}, if $\alpha_{\rm corrupt}\leq \alpha$ then 
 $n=\tilde O(d/\alpha^2 + \kappa^{1/2} d/ (\varepsilon \alpha))$ samples are sufficient for Alg.~\ref{alg:main} to achieve an error rate of $(1/\sigma)\|\hat w - w^*\|_\Sigma = \tilde O(\alpha)$ , where $\kappa:=\lambda_{\rm max}(\Sigma)/\lambda_{\rm min}(\Sigma)$.
\end{thm}

We focus on sub-Weibull distributions in the main text. A similar algorithm can be applied to the case where the noise in the samples is heavy-tailed, i.e., $k$-th moment bounded for $k\geq 4$. This results in an increased sample complexity of $n=\tilde{O}(d/\alpha^{2k/(k-1)} + \kappa^{1/2}d/(\varepsilon\alpha^{k/(k-1)}))$ to achieve the same level of error. We explain the heavy-tailed setting, provide detailed analysis and a proof, and discuss the results in App.~\ref{app:heavy_tailed}.
\begin{thm}[informal version of Coro.~\ref{coro:ht_robust2}]
\label{coro:ht_informal_robust} Alg.~\ref{alg:main_ht} is $(\varepsilon,\delta)$-DP. Under $(\Sigma,\sigma^2,w^*,K,a, \kappa_2, k)$-model of Assumption~\ref{asmp:distribution_ht} and $\alpha_{\rm corrupt}$-corruption of Assumption \ref{asmp:corrupt_ht}, if $1.2\alpha_{\rm corrupt}\leq \alpha^{k/(k-1)}$, then 
	$n=\tilde O( \kappa^{1/2} d/ (\varepsilon \alpha^{k/(k-1)})+d/\alpha^{2k/(k-1)}))$ samples are sufficient for Algorithm~\ref{alg:main_ht} to achieve an error rate of $(1/\sigma)\|\hat w - w^*\|_\Sigma = \tilde O(\alpha)$, where $\kappa:=\lambda_{\rm max}(\Sigma)/\lambda_{\rm min}(\Sigma)$.
\end{thm}

{\bf Contributions.} For a canonical problem of private linear regression under sub-Gaussian distributions, we provide a novel algorithm that achieves the state-of-the-art sample complexity and computational efficiency, improving upon the sample complexity of the prior efficient algorithms \cite{varshney2022nearly,cai2019cost} and nearly matching that of an exponential-time algorithm \cite{liu2022differential}.   
For the same problem, we show that the same algorithm is the first to achieve robustness against adversarial corruption of the response variables.   
Under a heavy-tailed distribution of the noise, we provide the first computationally efficient algorithm, to the best of our knowledge, that achieves a sample complexity close to that of an exponential-time algorithm of \cite{liu2022differential}. This algorithm is also the first to achieve robustness against adversarial corruption of the response variables, under heavy-tailed noise.

We start with the formal description of the setting in Sec.~\ref{sec:problem}, where we present the prior work of \citet{varshney2022nearly} and explain our main technical contributions. \citet{varshney2022nearly} propose a streaming version of DP-SGD with adaptive clipping. Streaming algorithm ensures independence between the current iterate and the current gradient, simplifying the analysis. Adaptive clipping finds the appropriate threshold to clip the norm of the gradients, which is an appropriate technique when there is no adversarial corruption.  However, these two algorithmic choices are sub-optimal. 

First, streaming DP-SGD can only use $O(n/\kappa)$ samples at each round, which increases the sensitivity and leads to an extra $\kappa^{1/2}$ factor in the sample complexity. Instead, we propose using a full-batch gradient descent and overcome the challenges in the analysis by relying on resilience (explained in Sec.~\ref{sec:sketch}). Together with the novel analysis technique we explain in Sec.~\ref{sec:standard}, this results in the gain of $\kappa^{1/2}$. 

Next, the gradient-norm clipping is vulnerable against label corruption.  Recall that a gradient is a product of the residual, $(w_t^\top x_i-y_i)$, and the covariate, $x_i$. An adversary can target those samples with small covariates and make big changes to the residuals, while managing to evade the clipping by the norm. Instead, we propose clipping the residual and the covariate separately. With adaptively estimated clipping thresholds, this provides robustness against label corruption.  

We present our main algorithm (Alg.~\ref{alg:main}) in 
Sec.~\ref{sec:main} with theoretical analyses and justification of the assumptions. We propose a novel adaptive clipping in Sec.~\ref{sec:adaptive}
We present numerical experiments on synthetic data that demonstrates the sample efficiency of our approach in 
Sec.~\ref{sec:exp}. 
We end with a sketch of our main proof ideas in
Sec.~\ref{sec:sketch}. 


\section{Problem formulation and background} 
\label{sec:problem} 



In linear regression without corruption, the following assumption is standard for the uncorrupted dataset $S_{\rm good}$, except for the fact that we assume   a more general family of  $(K,a)$-sub-Weibull distributions that recovers the standard sub-Gaussian family as a special case when $a=0.5$. 

\begin{asmp}[$(\Sigma,\sigma^2,w^*,K,a)$-model]
    \label{asmp:distribution}
	A multiset $S_{\rm good}= \{(x_i\in \reals^d, y_i\in \reals)\}_{i=1}^n$ of $n$ i.i.d.~samples  is from a linear model $y_i=\ip{x_i}{w^*}+z_i$, where the input vector $x_i$ is zero mean, $\E[x_i]=0$, with a positive definite covariance  $\Sigma:=\E[x_ix_i^\top]\succ 0$, and the (input dependent) label noise $z_i$ is zero mean, $\E[z_i]=0$, with variance $\sigma^2 := \E[z_i^2]$.
	We further assume $\E[x_iz_i]=0$, which is equivalent to assuming that the true parameter $w^* = \Sigma^{-1} {\mathbb E}[y_ix_i]$. 
	We assume that the marginal distribution of $x_i$ is $(K,a)$-sub-Weibull and that of $z_i$ is also $(K,a)$-sub-Weibull, as defined below.  
	\end{asmp}
	
	Sub-Weibull distributions provide Gaussian-like tail bounds determining the resilience of the dataset in Lemma~\ref{lemma:subweibull_res_conditions}, which our analysis critically relies on and whose necessity is justified in Sec.~\ref{sec:lb}. 

\begin{definition}[sub-Weibull distribution {\citep
{kuchibhotla2018moving}}
]
\label{def:a_tail}
For some $K,a>0$, we say a random vector $x\in \reals^d$ is from a $(K,a)$-sub-Weibull distribution if for all $v\in \reals^d$, 
$\E\left[\exp\left(\left(\frac{\ip{v}{x}^2}{K^2\E[\ip{v}{x}^2]}\right)^{1/(2a)}\right)\right]\leq 2$.
\end{definition}

Our goal is to estimate the unknown parameter $w^*$, given upper bounds on the sub-Weibull parameters, $(K,a)$, and a corrupted dataset under the the standard definition of {\em label corruption} in \citep{bhatia2015robust}. There are variations in literature on the definition, which we survey in App.~\ref{app:related}. 

\begin{asmp}[$\alpha_{\rm corrupt}$-corruption]  \label{asmp:corrupt} 
Given a dataset $S_{\rm good}=\{(x_i, y_i)\}_{i=1}^n$, an adversary inspects all the data points, selects $\alpha_{\rm corrupt} n$ data points denoted as $S_r$, and replaces the labels with arbitrary labels while keeping the covariates unchanged. 
We let $S_{\rm bad}$ denote this set of $\alpha_{\rm corrupt} n$ newly labelled examples by the adversary. Let the resulting set be $S:=S_{\rm good}\cup S_{\rm bad}\setminus S_r$.  We further assume that the corruption rate is bounded by $\alpha_{\rm corrupt} \leq \bar \alpha$, where $\bar\alpha$ is a known positive constant satisfying $\bar\alpha\leq 1/10$, $72C_2\,K^2 \, \bar{\alpha}\,\log^{2a}(1/(6\bar{\alpha})) \log(\kappa)  \le 1/2$, and $ 2C_2K^2\log^{2a}(1/(2\bar{\alpha}))\geq 1 $ for the $(K,a)$-sub-Weibull distribution of interest and a positive constant $C_2$ defined in Lemma~\ref{lemma:subweibull_res_conditions} that only depends on $(K,a)$. 
\end{asmp}

\medskip\noindent{\bf Notations.}  A vector $x\in{\mathbb R}^d$ has  the Euclidean norm  $\|x\|$. For a matrix $M$, we use $\|M\|_2$ to denote the spectral norm. The error is measured in $\|\hat w-w^*\|_\Sigma :=\|\Sigma^{1/2}(\hat w - w^*)\|$ for some PSD matrix $\Sigma$. The identity matrix is denoted by ${\bf I}_d\in{\mathbb R}^{d\times d}$. Let  $[n]=\{1,2,\ldots,n\}$. $\tilde{O}(\cdot)$ hides some constants terms, $K,a=\Theta(1)$, and poly-logarithmic terms in $n$, $d$, $1/\varepsilon$, $\log(1/\delta)$, $1/\zeta$, and $1/\alpha_{\rm corrupt}$. For a vector $x\in \reals^d$, we define ${\rm clip}_a(x):=x\cdot \min\{1, \frac{a}{\|x\|}\}$.

\medskip\noindent{\bf Background on DP.} 
Differential Privacy   is a standard measure of privacy leakage when data is accessed via queries, introduced by \citet{dwork2006calibrating}. 
Two datasets $S$ and $S'$ are said to be neighbors if they differ at most by one entry, which is denoted by $S\sim S'$. A stochastic query $q$ is said to be $(\varepsilon,\delta)$-differentially private for some $\varepsilon>0$ and $\delta\in [0,1]$, if ${\mathbb P}(q(S)\in A) \leq e^\varepsilon {\mathbb P}(q(S)\in A)  + \delta$, for all neighboring datasets $S\sim S'$ and all subset $A$ of the range of the query. 
We build upon two widely used DP primitives, the Gaussian mechanism and the private histogram. 
A central concept in DP mechanism design is the {\em sensitivity} of a query, defined as $\Delta_q:=\sup_{S\sim S'} \|q(S)-q(S')\|$. We describe Gaussian mechanism and private histogram  in App.~\ref{app:dp}.

\subsection{Comparisons with the prior work}
\label{sec:standard}

When there is no adversarial corruption, the state-of-the-art approach introduced by \citet{varshney2022nearly} is based on stochastic gradient descent, where privacy is ensured by gradient norm clipping and the Gaussian mechanism to ensure privacy. There are two main components in this approach: adaptive clipping and streaming SGD. 
Adaptive clipping with an appropriate threshold $\theta_t$ ensures that no data point is clipped while providing a bound on the sensitivity of the average mini-batch gradient, which ensures we do not add too much noise. 
The streaming approach, where each data point is only touched once and discarded,  ensures independence between the past  iterate $w_{t}$ and the gradients at round $t+1$, which the analysis critically relies on. 
For $T=\tilde \Theta(\kappa)$ iterations where $\kappa$ is the condition number of the covariance $\Sigma$ of the covariates, the dataset $S=\{(x_i,y_i)\}_{i=1}^n$ is partitioned into $\{B_t\}_{t=0}^{T-1}$ subsets of equal size: $|B_t|=\tilde\Theta(n/\kappa)$. 
At each round $t <T$, the gradients are clipped and averaged with additive Gaussian noise chosen to satisfy $(\varepsilon,\delta)$-DP: 
\begin{eqnarray} 
    \label{eq:stream1}
    w_{t+1} \; \gets \; w_{t}-\eta\Big( \frac{1}{|B_t|} \sum_{i\in B_t} {\rm clip}_{\theta_t} (x_i(w_t^\top x_i - y_i )) + \frac{\theta_t \sqrt{2\log(1.25/\delta)}}{\varepsilon |B_t|} \nu_t\Big)\;,
\end{eqnarray}
where $\nu_t\sim \cN(0,{\bf I}_d)$.
In \citet{varshney2022nearly}, a variation of this streaming SGD requires $n=\tilde O(\kappa^2d/\varepsilon + d/\alpha^2 + \kappa d/(\varepsilon\alpha))$ to achieve an error of $\| w_T - w^* \|_\Sigma^2=O(\sigma^2\alpha^2)$.

{\bf Our technical innovations.}
Our approach builds upon such gradient based methods but makes several important innovations. First, we use full-batch gradient descent, as opposed to the streaming SGD above.
Using all $n$ samples reduces the sensitivity of the per-round gradient average, since $n>|B_t|=\tilde{\Theta}(n/\kappa)$. 
This improves the sample complexity to $n=\tilde O(d/\alpha^2 + \kappa^{1/2}d/(\varepsilon\alpha))$ to achieve an error of $\| w_T - w^* \|_\Sigma^2=O(\sigma^2\alpha^2)$. However, full-batch GD loses the independence that the streaming SGD enjoyed between $w_{t}$ and the samples used in the round $t+1$. This dependence makes the analysis more challenging. We instead propose using the {\em resilience} property of sub-Weibull distributions to precisely track the bias and variance of the (dependent) full-batch gradient average. Resilience   is a central concept in robust statistics which we explain in Sec.~\ref{sec:sketch}. 

Next, one critical component in achieving this improved sample complexity is the new analysis technique we introduce for tracking the end-to-end gradient updates. Since our gradient descent algorithm is not guaranteed to make progress every step, we can not use the vanilla one-step ahead analysis. Taking the full end-to-end analysis by expanding the whole gradient trajectory will introduce too many correlated cross-terms which are very hard to control. Therefore, we leverage an every $\kappa$-step analysis and show that the objective function at least decreases geometrically every $\kappa$ steps. To be more specific, our analysis technique in  App.~\ref{app:proof_main} steps 3 and 4 opens up the iterative updates from beginning to end, and exploits the fact that $\lambda_{\rm max}((\eta\Sigma)^{1/ 2} (1-\eta \Sigma)^i(\eta \Sigma)^{1/ 2})$ is upper bounded by $1/(i+1)$ when $\|\eta\Sigma\|\leq 1$. This technique is critical in achieving the near-optimal dependence in $\kappa$. This might be of independent interest to other analysis of gradient-based algorithms. We refer to the beginning of step 3 in App.~\ref{app:proof_main} for a detailed explanation.

Finally, we propose a novel clipping method that separately clips $x_i$ and $(w_t^\top x_i-y_i)$ in the gradient. This is critical in achieving robustness to label-corruption, as we explain in detail in Sec.~\ref{sec:algortihm}.


\section{Robust and DP linear regression}
\label{sec:main}

We introduce a  gradient descent approach for linear regression with a novel adaptive clipping that ensures robustness against label-corruption. This achieves a near-optimal sample complexity and, for the special case of private linear regression without adversarial corruption,  improves upon the  state-of-the-art algorithm. 

\subsection{Algorithm} 
\label{sec:algortihm} 
The skeleton of our approach in Alg.~\ref{alg:main} is the general DP-SGD \citep{abadi2016deep,song2013stochastic} with   adaptive clipping \citep{andrew2021differentially}. However, the standard adaptive clipping  is not robust against label-corruption under the more general $(K,a)$-sub-Weibull assumption. In particular, it is possible under sub-Weibull distribution that a positive fraction of the covariates are close to the origin, which is not possible under Gaussian data due to concentration. 
In this case, the adversary can choose  to corrupt those  points with small norm,  $\|x_i\|$, making large changes in the residual, $(y_i-w_t^\top x_i)$, while evading
the standard clipping (by the norm of the gradient), since the norm of the gradient, $\|x_i(y_i-w_t^\top x_i)\|=\|x_i\|\,|y_i-w_t^\top x_i|$, can remain under the threshold.  
This is problematic, since the bias due to the corrupted samples in the gradient scales proportional to the magnitude of the residual (after clipping). 
To this end, we propose clipping the norm and the residual separately: 
${\rm clip}_{\Theta}(x_{i}){\rm clip}_{\theta_t}\left(w_t^\top x_i-y_i\right)$. This keeps the sensitivity of gradient average bounded by $\Theta\theta_t$, and the subsequent Gaussian mechanism in line~\ref{line:sgd} ensures $(\varepsilon_0,\delta_0)$-DP at each round. Applying advanced composition in Lemma~\ref{lem:composition} of $T$ rounds, this ensures end-to-end $(\varepsilon,\delta)$-DP.

{\bf Novel adaptive clipping.} 
When clipping with ${\rm clip}_{\Theta}(x_{i})$, the only purpose of clipping the covariate by its norm, $\|x_i\|$, is to bound the sensitivity of the resulting clipped gradient. In particular, we do not need to make it robust as there is no corruption in the covariates. Ideally, we want to select the smallest threshold $\Theta$ that does not clip any of the covariates.  Since the norm of a covariate is  upper bounded by $\|x_i\|^2 \leq K^2 {\rm Tr}(\Sigma) \log^{2a}(1/\zeta)$ with probability $1-\zeta$ (Lemma~\ref{lemma:norm_a_tail}),  we estimate the unknown ${\rm Tr}(\Sigma)$ using Private Norm Estimator in Alg.~\ref{alg:norm} in App.~\ref{app:norm} and set the norm threshold $\Theta=K\sqrt{2\Gamma} \log^a(n/\zeta)$ (Alg.~\ref{alg:main} line~\ref{line:clip0}). The  $n$ in the logarithm ensures that the union bound holds. 

When clipping with ${\rm clip}_{\theta_t} (w_t^\top x_i-y_i )$, the purpose of clipping the residual by its magnitude, $|y_i-w_t^\top x_i|=|(w^*-w_t)^\top x_i + z_i|$, is to bound the sensitivity of the gradient and also to provide robustness against label-corruption. We want to choose a threshold that only clips corrupt data points and at most a few clean data points. In order to achieve an error $(1/\sigma) \|w_T-w^*\|_\Sigma = O(\alpha)$, we know that any set of $(1-\alpha)$ fraction of the clean data points is sufficient to get a good estimate of the average gradient, and we can find such a large enough set of points that satisfy  $|(w^*-w_t)^\top x_i + z_i|^2 \leq  (\|w_t-w^*\|_\Sigma^2 + \sigma^2) C K^2 \log^{2a}(1/(2\alpha))$ from  
Lemma~\ref{lemma:norm_a_tail}.  
At the same time, this threshold on the residual is small enough to guarantee robustness against the label-corrupted samples. 
We introduce the robust and private Distance Estimator in Alg.~\ref{alg:distance} to estimate the unknown (squared and shifted) distance, $\|w_t-w^*\|_\Sigma^2 + \sigma^2$, and set the distance threshold $\theta_t=2\sqrt{2\gamma_t}\sqrt{9C_2K^2\log^{2a}(1/(2\alpha))}$ (Alg.~\ref{alg:main} line~\ref{line:clip}). 
Both norm and distance estimation rely on private histogram (Lemma~\ref{lem:hist-KV17}), but over a set of statistics computed on  partitioned datasets, which we explain in detail in Sec.~\ref{sec:adaptive}. 

\begin{algorithm2e}  
   \caption{Robust and Private  Linear Regression}
   \label{alg:main}
   	\DontPrintSemicolon 
	\KwIn{ $S=\{(x_i, y_i)\}_{i=1}^{3n}$, DP parameters  $(\varepsilon ,\delta )$, $T$, learning rate $\eta$, failure probability $\zeta$,  target error  $\alpha$, distribution parameter $(K,a)$}
	\SetKwProg{Fn}{}{:}{}
	{ 
	Partition dataset $S$ into three equal sized disjoint subsets $S =  S_1\cup S_2\cup S_3  $. \\
	$\delta_0\gets\frac{\delta}{2T}$, $\varepsilon_0\gets\frac{\varepsilon}{4\sqrt{T\log(1/\delta_0)}}$, $\zeta_0\gets \frac{\zeta}{3}$, 	$w_0\gets 0$\\
	$\Gamma\gets {\rm Private Norm Estimator}(S_1, \varepsilon_0, \delta_0, \zeta_0)$ 
 \tcc{using Algorithm~\ref{alg:norm}, Appendix~\ref{app:norm}}
    $\Theta\gets K\sqrt{2\Gamma}\log^{a}(n/\zeta_0)$\label{line:clip0}\\
	\For{$t=0,1, 2, \ldots, T-1$}{ 
	$ \gamma_t \gets {\rm Robust Private Distance Estimator}(S_2, w_t, \varepsilon_0, \delta_0, \alpha, \zeta_0)$ 
  \tcc{using Algorithm~\ref{alg:distance}}
	$\theta_t \gets 2 \sqrt{2\gamma_t}  \cdot\sqrt{9C_2K^2\log^{2a}(1/(2\alpha ) )}$.\label{line:clip}\\
	Sample $\nu_t \sim \cN\left(0, \mathbf{I}_d\right)$\\
 $\tilde{g}_i^{(t)}\gets {\rm clip}_{\Theta}(x_i){\rm clip}_{\theta_t}(x_i^\top w_t-y_i)$\\
 $\phi_t=(\sqrt{2\log(1.25/\delta_0)}\Theta\theta_t)/(\varepsilon_0 n)$\\
	$w_{t+1}\gets   w_{t}-\eta\left(\frac{1}{n}\sum_{i\in S_3}\tilde{g}_i^{(t)}+\phi_t\nu_t\right)	$ \label{line:sgd}\\
	}
	Return $w_T$
	}
\end{algorithm2e}
\subsection{Analysis} 
\label{sec:analysis}
We show that Algorithm~\ref{alg:main} achieves a near-optimal sample complexity. We provide a proof in Appendix~\ref{app:proof_main} and a sketch of the proof  in Section~\ref{sec:sketch}. We address the necessity of the assumptions in Sec.~\ref{sec:lb}, along with some lower bounds.

\begin{thm}
    \label{thm:main}
    Algorithm~\ref{alg:main} is $(\varepsilon, \delta)$-DP. 
Under $(\Sigma,\sigma^2,w^*,K,a)$-model of Assumption~\ref{asmp:distribution} and $\alpha_{\rm corrupt}$-corruption of Assumption \ref{asmp:corrupt} and for any failure probability $\zeta\in(0,1)$ and target error rate $\alpha \geq \alpha_{\rm corrupt}$, if the sample size is large enough such that 
	\begin{align}
	    n = 
	    \tilde O  \left(K^2d \log^{2a+1}\Big(\frac{1}{\zeta}\Big)+\frac{d+\log(1/\zeta)}{\alpha ^2}
	    +
	    \frac{K^2d T^{1/2}\log(\frac{1}{\delta})\log^{a}(\frac{1}{\zeta})}{\varepsilon \alpha}  \right) ,
	    \label{eq:main_n} 
	\end{align} 
	with a large enough constant where $\tilde O$ hides poly-logarithmic terms in $d$, $n$, and $\kappa$, then the choices of a  step size $ \eta =  1/(C\lambda_{\max}(\Sigma))$ for any $C \geq 1.1$ and the number of iterations,  
	    $T= \tilde\Theta\left(\kappa \log\left(\|w^*\|
	    \right)\right)\,$ for a condition number of the covariance $\kappa:=\lambda_{\rm max}(\Sigma)/\lambda_{\rm min}(\Sigma)$, 
ensures that, with probability $1-\zeta$, Algorithm~\ref{alg:main} achieves  
	\begin{align}
		 &\E_{\nu_1, \cdots, \nu_t\sim \cN(0, \mathbf{I}_d)}\big[\,\|  w_{T}-w^*\|_\Sigma^2\,\big] =
		 \tilde O\Big(\, K^4\sigma^2 \alpha^2\log^{4a}\Big(\frac{1}{\alpha}\,\Big)\,\Big)\;,
		 \label{eq:main} 
	\end{align}
	where the expectation is taken over the noise added for DP, and $\tilde\Theta(\cdot)$ hides logarithmic terms in $K,\sigma,d,n,1/\varepsilon,\log(1/\delta),1/\alpha$, and $\kappa$. 
\end{thm}

{\bf Optimality.} Omitting some constant and logarithmic terms, Alg.~\ref{alg:main} requires 
\begin{eqnarray} 
n&=&\tilde O\Big( \,\frac{d}{\alpha_{}^2} +  \frac{\kappa^{1/2} d \log(1/\delta)}{\varepsilon\alpha_{}}   \,\Big)\;, \label{eq:main_n_simple}
\end{eqnarray} 
samples to ensure an error rate of ${\mathbb E}[\|w_{T}-w^*\|_\Sigma^2] = \tilde O(\sigma^2 \alpha^2) $ for any $\alpha\geq \alpha_{\rm corrupt}$. 
The lower bound on the achievable error of $\sigma^2\alpha^2 \geq \sigma^2\alpha^2_{\rm corrupt}$  is due to the label-corruption and cannot be improved, as it matches an information theoretic lower bound we provide in Proposition~\ref{propo:lb}. 
In the special case when  the covariate follows  a sub-Gaussian distribution, that is $(K, 1/2)$-sub-Weibull for a constant $K$, there is an $n=\Omega(d/\alpha^2 + d/(\eps\alpha))$ lower bound (\cite{cai2019cost}, Theorem 4.1), and our upper bound matches this lower bound up to a factor of $\kappa^{1/2}$ in the second term and other logarithmic factors. \eqref{eq:main_n_simple} is the best known rate among all efficient private linear regression algorithms, strictly improving upon existing methods when $\log(1/\delta)=\tilde{O}(1)$. We discuss some exponential time algorithms that closes the $\kappa^{1/2}$ gap in Sec.~\ref{sec:lb}.

{\bf Comparisons with the  state-of-the-art.} The best existing efficient algorithm by \cite{varshney2022nearly} can only handle the case where there is {\em no adversarial corruption}, and requires
$n=\tilde O(\kappa^2 d \sqrt{\log(1/\delta)} /\varepsilon + d/\alpha^2 + \kappa d \sqrt{\log(1/\delta)}/(\varepsilon\alpha))$ to achieve an error rate of $\sigma^2 \alpha^2$. Compared to   \eqref{eq:main_n_simple}, the first term dominates in its dependence in $\kappa$, which is a factor of $\kappa$ larger than \eqref{eq:main_n_simple}. The third term is larger by a factor of $\kappa^{1/2}$ but smaller by a factor of $\log^{1/2} (1/\delta)$, compared to the second term in  \eqref{eq:main_n_simple}. 

In the {\em non-private case}, when $\varepsilon=\infty$, a recent line of work has developed algorithms for linear regression that are robust to label corruptions~\citep{bhatia2015robust, bhatia2017consistent, suggala2019adaptive, dalalyan2019outlier}. Of these, \cite{bhatia2015robust, dalalyan2019outlier} are relevant to our work as they consider the same adversary model as Assumption~\ref{asmp:corrupt}. When $x_i$'s and $z_i$'s are sampled from $\mathcal{N}(0, \Sigma)$ and $\mathcal{N}(0,\sigma^2)$, \citet{dalalyan2019outlier} proposed a Huber loss based estimator that achieves error rate of $ \sigma^2 \alpha^2\log^2(n/\delta)$  when $n = \tilde{O}\left(\kappa^2 d/\alpha^2\right)$. Under the same setting, \citet{bhatia2015robust} propoased a hard thresholding based estimator that achieves $\sigma^2 \alpha^2$ error rate with $\tilde{O}\left(d/\alpha^2\right)$ sample complexity. Our results in Theorem~\ref{thm:main} match these rates, except for the sub-optimal dependence on $\log^{4a}(1/\alpha).$ Another line of work considered both label and covariate corruptions and developed optimal algorithms for parameter recovery~\citep{diakonikolas2019efficient,diakonikolas2019sever, prasad2018robust, pensia2020robust, cherapanamjeri2020optimal,jambulapati2020robust, klivans2018efficient,bakshi2021robust, zhu2019generalized, depersin2020spectral}. The best existing efficient algorithm , e.g. \citep{pensia2020robust}, achieves error rate of $\sigma^2\alpha^2\log(1/\alpha)$ when $n= \Tilde{O}\left(d/\alpha^2\right)$, and the  $x_i$ and $z_i$ are  sampled from $\mathcal{N}(0, I)$ and  $\mathcal{N}(0,\sigma^2)$. 

Under both privacy requirements and adversarial corruption, the only algorithm with a provable guarantee is an exponential time approach, known as High-dimensional Propose-Test-Release (HPTR), of {\citet[Corollary C.2]{liu2022differential}}, which achieves a  sample complexity of $n=O(d/\alpha^2 + (d+\log(1/\delta))/ (\varepsilon \alpha))$. Notice that there is no dependence on $\kappa$ and the $\log(1/\delta)$ term scales as $1/(\varepsilon \alpha)$ as opposed to $\kappa d^{1/2}/(\varepsilon \alpha)$ of \eqref{eq:main_n_simple}. It remains an open question if {\em computationally efficient} private linear regression algorithms can achieve such a $\kappa$-independent sample complexity. 
Further, HPTR is robust against a stronger adversary who corrupts the covariates also and not just the labels. Under this stronger adversary, it remains open if there is an efficient algorithm that achieves $n=O(d/\alpha^2+d/(\varepsilon\alpha))$ sample complexity  even for constant $\kappa$ and $\delta$.


\subsection{Lower bounds}
\label{sec:lb}

{\bf Necessity of our assumptions.} 
A tail assumption on the covariate $x_i$ such as Assumption~\ref{asmp:distribution} is necessary to achieve $n=O(d)$ sample complexity in \eqref{eq:main_n_simple}. Even when the covariance $\Sigma$ is close to identity, without further assumptions on the tail of covariate $x$, the result in~\citep{bassily2014private} implies that for $\delta<1/n$ and sufficiently large $n$, no $(\varepsilon,\delta)$-DP estimator can achieve excess risk $\|\hat w - w^*\|^2_\Sigma$ better than $\Omega(d^3/(\varepsilon^2n^2))$ (see Eq.~(3) in \citep{wang2018revisiting}). Note that this lower bound is a factor $d$ larger than our upper bound that benefits from the additional tail assumption.

A tail assumption on the noise $z_i$ such as Assumption~\ref{asmp:distribution} is necessary to achieve $n=O(d/(\varepsilon\alpha))$ dependence on the sample complexity in \eqref{eq:main_n_simple}. For heavy-tailed noise, such as $k$-th moment bounded noise, the dependence can be significantly larger. {\citet[Proposition C.5]{liu2022differential}} implies that for $\delta=e^{-\Theta(d)}$ and $4$-th moment bounded $x_i$ and $z_i$, any $(\varepsilon,\delta)$-DP estimator requires  $n=\Omega(d/(\varepsilon\alpha^2))$, which is a factor of  $1/\alpha$ larger, to achieve excess risk ${\mathbb E}[\|\hat w - w^*\|_\Sigma^2 ] = \tilde O(\sigma^2\alpha^2)$.

The assumption that only label is corrupted is critical for Algorithm~\ref{alg:main}. 
The average of the clipped gradients can be significantly more biased, if the adversary can place the covariates of the corrupted samples in the same direction. In particular, the bound on the bias of our gradient step  in \eqref{eq:ut3} in  App.~\ref{app:proof_main} would no longer hold.  
Against such strong attacks, one requires additional steps to estimate the mean of the gradients robustly and privately, similar to those used in robust private mean estimation \citep{liu2021robust,kothari2021private,hopkins2022efficient,ashtiani2022private}.  This is outside the scope of this paper.

{\bf Lower bounds under label  corruption.} 
Under the $\alpha_{\rm corrupt}$ label corruption setting (Assumption~\ref{asmp:corrupt}), even with infinite data and without privacy constraints, no algorithm is able to learn $w^*$ with $\ell_2$ error better than $\alpha_{\rm corrupt}$. We provide a formal derivation  for completeness. 

\begin{propo}\label{propo:lb}
Let $\cD_{\Sigma,\sigma^2, w^*,K,a}$ be a class of joint distributions on $(x_i, y_i)$ from $(\Sigma,\sigma^2, w^*,K,a)$-model in Assumption~\ref{asmp:distribution}. Let $S_{n,\alpha}$ be an $\alpha$-corrupted dataset of $n$ i.i.d. samples from some distribution $\cD\in \cD_{\Sigma,\sigma^2, w^*,K,a}$ under Assumption~\ref{asmp:corrupt}. Let $\cM$ be a class of estimators that are functions over the datasets $S_{n, \alpha}$. Then there exists a positive constant $c$ such that 
$$\min_{n, \hat{w}\in \cM} \; \;\max_{S_{n,\alpha}, \cD\in \cD_{\Sigma,\sigma^2, w^*,K,a}, w^*, K , a,} \;\; \E[\|\hat{w}-w^*\|_\Sigma^2] \geq  c\,  \alpha^2\, \sigma^2$$. 
\end{propo}

A proof is provided in Appendix~\ref{sec:label-lower-bounds}. A similar lower bound can be found in {\citet[Theorem 6.1]{bakshi2021robust}}.

\section{Adaptive clipping for the gradient norm} 
\label{sec:adaptive}

In the ideal clipping thresholds for the norm and the residual, there are unknown terms which we need to estimate adaptively, 
$(\|w_t-w^*\|_{\Sigma}^2 + \sigma^2)$ and ${\rm Tr}(\Sigma)$, up to  constant multiplicative errors. 
 We privately estimate the (squared and shifted) distance to optimum, $(\|w_t-w^*\|_{\Sigma}^2 + \sigma^2)$, with Alg.~\ref{alg:distance} and privately estimate the average input norm, ${\mathbb E}[\|x_i\|^2] = {\rm Tr}(\Sigma) $, with Alg.~\ref{alg:norm} in App.~\ref{app:norm}.
These are used to get the clipping thresholds in Alg.~\ref{alg:main}. 
We propose a trimmed mean approach below for distance estimation.  
The norm estimator is similar and is provided in App.~\ref{app:norm}.

\medskip\noindent{\bf Private distance estimation using private trimmed mean.} 
The goal is to estimate the (shifted) distance to optimum, $\|w_t-w^* \|_\Sigma^2 + \sigma^2$, up to some constant multiplicative error. 
Note that this is precisely the task of estimating the variance of the residual $b_i = y_i-w_t^\top x_i$. When there is no adversarial corruption and no privacy constraint, we can simply use the empirical variance estimator $(1/n) \sum_{i\in[n]}(y_i-w_t^\top x_i)^2$ to obtain a good estimate. However, the empirical variance estimator is not robust against adversarial corruptions since one outlier can make the estimate arbitrarily large. A classical idea is using the \textit{trimmed  estimator} from~\citep{tukey1963less}, which throws away the $2\alpha$ fraction of residuals $b_i$ with the largest magnitude. For datasets with resilience property as assumed in this paper, this will guarantee an accurate estimate of the distance to optimum in the presence of $\alpha$ fraction of corruptions.

To make the estimator private, it is tempting to simply add a Laplacian noise to the estimate. However, the sensitivity of the trimmed estimator is unknown and depends on the distance to the optimum that we aim to estimate; this makes it challenging to determine  the  variance of the Laplacian noise we add. Instead, we propose to partition the dataset into $k$ batches, compute an estimate for each batch, and form a histogram with over those $k$ estimates. Using a private histogram mechanism with geometrically increasing bin sizes, we propose using the bin with the most estimates to guarantee a constant factor approximation of the distance to the optimum. We describe the algorithm as follows.




\begin{algorithm2e}  
   \caption{Robust and Private Distance Estimator} 
   \label{alg:distance} 
   	\DontPrintSemicolon 
	\KwIn{$S_2 = \{(x_i, y_i)\}_{i=1}^n$, current  $w_t$,   $(\varepsilon_0,\delta_0)$,  $\bar\alpha$,  $\zeta$}
	\SetKwProg{Fn}{}{:}{}
	{
	Let $b_i\gets (y_i-w_t^\top x_i)^2$,  $\forall i\in[n]$ and  $\tilde{S}\gets \{b_i\}_{i=1}^n$.\\
	Partition $\tilde{S}$ into $k=\lceil C_1\log(1/(\delta_0\zeta))/\varepsilon_0 \rceil$ subsets of equal size and let $G_j$ be the $j$-th partition.\\
	For $j\in [k]$, denote $\psi_j$ as the $(1-3\bar{\alpha})$-quantile of $G_j$ and  $\phi_j \gets \frac{1}{|G_j|} \sum_{i\in G_j} b_i\mathbf{1}\{b_i\le \psi_j\}$.\\	
	Partition $[0, \infty)$ into  geometrically increasing intervals $\Omega:= \left\{\ldots,\left[2^{-1}, 1\right),\left[ 1, 2\right),\left[2,2^2\right), \ldots\right\} \cup\{[0,0]\}$   \\
	
	Run $(\varepsilon_0, \delta_0)$-DP histogram  of Lemma~\ref{lem:hist-KV17} on $\{ \phi_j\}_{j=1}^k$ over $\Omega$ \\
	{\bf if} all the bins are empty {\bf then} 
	Return $\perp$\\
	Let $[\ell, r]$ be a non-empty bin that contains the maximum number of points in the DP histogram\\
	\Return  $ \ell$
	} 
\end{algorithm2e}
This algorithm gives an estimate of the distance up to a constant multiplicative error as we show in the following theorem. 
We provide a proof in App.~\ref{app:proof_distance}. 
\begin{thm}
\label{thm:distance}
	 Algorithm~\ref{alg:distance} is $(\varepsilon_0, \delta_0)$-DP. 
	 For an $\alpha_{\rm corrupt}$-corrupted dataset $S_2$ and an upper bound $\bar\alpha$ on $\alpha_{\rm corrupt}$ that satisfy  Assumption~\ref{asmp:distribution} and  $37C_2K^2\cdot \bar{\alpha}\log^{2a}(1/(6\bar{\alpha}))\le 1/4$ and any $\zeta\in (0,1)$, if 
    \begin{align}
n=O\left(\frac{(d+\log((\log(1/(\delta_0\zeta)))/\varepsilon_0\zeta))(\log(1/(\delta_0\zeta)))}{\bar{\alpha}^2 \varepsilon_0}
		\right), \label{eq:distance}
	\end{align} 
	with a large enough constant then, with probability $1-\zeta$,   Algorithm~\ref{alg:distance} returns $\ell$ such that 
$
	\frac{1}{4}(\|w_t-w^*\|_{\Sigma}^2+\sigma^2) \;\leq\; \ell\; 
	\leq\; 
	 4(\|w_t-w^*\|_{\Sigma}^2+\sigma^2)
	 $.
\end{thm}
Note that in Theorem~\ref{thm:distance}, we only need to estimate distance up to a constant multiplicative error, as opposed to an error that depends on our final end-to-end desired level $\alpha$. Consequently, we require smaller sample complexity (that doesn’t depend on $\alpha$) than other parts of our approach.
\begin{remark}
	While DP-STAT (Algorithm 3 in \cite{varshney2022nearly}) can also be used to estimate $\|w_t-w^*\|_\Sigma+\sigma$ (and it would not change the ultimate sample complexity in its dependence on $\kappa$, $d$, $\varepsilon$, and $n$), there are three important improvements we make: $(i)$ DP-STAT requires the knowledge of $\|w^*\|_\Sigma+\sigma$; $(ii)$ our utility guarantee has improved dependence in $K$ and $\log^{2a}(n)$; and $(iii)$  Algorithm~\ref{alg:distance} is robust against label corruption.
\end{remark}

{\bf Upper bound on clipped good data points.} Using the above estimated distance to the optimum in selecting a threshold $\theta_t$, we also need to ensure that we do not clip too many clean data points. The tolerance in our algorithm to reach the desired level of accuracy is clipping $O(\alpha)$ fraction of clean data points. This is ensured by the following lemma, 
and we provide a proof in Appendix~\ref{app:proof_clipping_fraction}.

\begin{lemma}
\label{lemma:clipping_fraction}
Under Assumption~\ref{asmp:distribution} and for all $t\in [T]$, if $
	\theta_t  \geq \sqrt{9C_2K^2\log^{2a}(1/(2\alpha))} \cdot \left(\|w^*-w_t\|_\Sigma+\sigma\right)$ 
then  
$
	\left|\left\{i\in S_3\cap S_{\rm good}: \left|w_t^\top x_i-y_i\right|\geq \theta_t\right\}\right|  \leq  \alpha n$.

\end{lemma}

\section{Experimental results} 
\label{sec:exp}

\begin{figure*}
    \centering
    \includegraphics[scale=0.25]{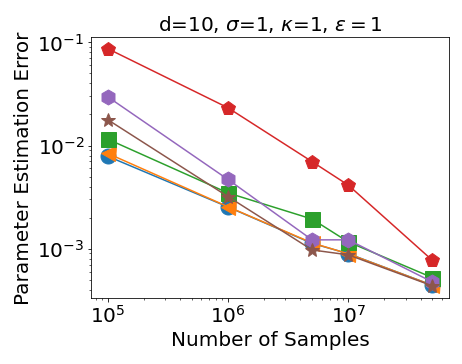}
    \includegraphics[scale=0.25]{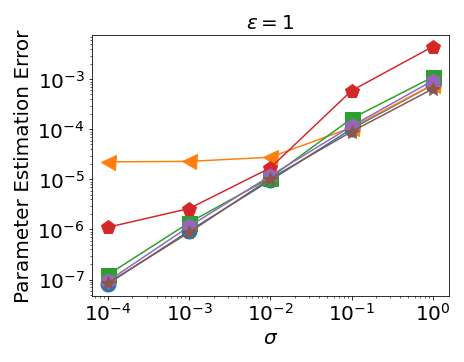}
    \includegraphics[scale=0.25]{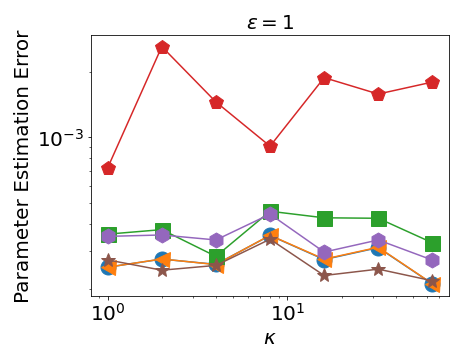}
    \includegraphics[scale=0.3]{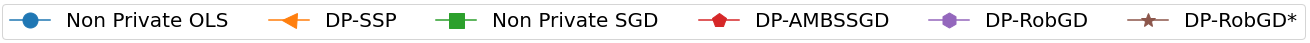}
    
    \vspace{-0.1in}\caption{Performance of various techniques on DP linear regression. $d=10$ in all the experiments. $n=10^7, \kappa=1$ in the $2^{nd}$ experiment. $n=10^7, \sigma=1$ in the $3^{rd}$ experiment, where $\kappa$ is the condition number of  $\Sigma$ and $\sigma^2$ is the variance of the label noise $z_i$.} 
    \label{fig:dp_regression}
    \vspace{-0.1in}
\end{figure*}

  \subsection{DP Linear Regression}
We present experimental results comparing our proposed technique ($\dprobgd$) with other baselines. We  consider non-corrupted regression in this section and defer corrupted regression to the App.~\ref{appendix:experiments}. We begin by describing the problem setup and the baseline algorithms first.  \\
\textbf{Experiment Setup.} We generate data for all the experiments using the following generative model. The parameter vector $w^*$ is uniformly sampled from the surface of a unit sphere. The covariates $\{x_i\}_{i=1}^n$ are first sampled from $\mathcal{N}(0, \Sigma)$ and then projected to unit sphere. We consider diagonal covariances $\Sigma$ of the following form: $\Sigma[0,0] = \kappa$, and $\Sigma[i,i] = 1$ for all $i\geq 1.$  Here $\kappa\geq 1$ is the condition number of $\Sigma$. We generate noise $z_i$ from uniform distribution over $[-\sigma, \sigma].$ Finally, the response variables are generated as follows $y_i =  x_i^\top  w^* + z_i.$ All the experiments presented below are repeated $5$ times and the averaged results are presented.  We set the DP parameters $(\epsilon, \delta)$ as $\epsilon = 1, \delta = \min(10^{-6}, n^{-2}).$ Experiments for $\epsilon=0.1$ can be found in  Fig.~\ref{fig:dp_regression_more} in the App.~\ref{appendix:experiments}.
\\
\textbf{Baseline Algorithms.} We compare our estimator with the following baseline algorithms: 
\begin{itemize}[leftmargin=*]
    \vspace{-0.2cm}\item \textit{Non private algorithms:} ordinary least squares ($\ols$), one-pass stochastic gradient descent with tail-averaging ($\sgd$). For $\sgd$, we use a constant step-size of $1/(2\lambda_{\text{max}})$ with $n/T$ minibatch size, where $T = 3\kappa \log{n}$. 
    \vspace{-0.2cm}\item \textit{Private algorithms:} sufficient statistics perturbation ($\ssp$)~\citep{foulds2016theory, wang2018revisiting}, differentially private stochastic gradient descent ($\dpsgd$)~\citep{varshney2022nearly}. $\ssp$ had the best empirical performance among numerous techniques studied by \citet{wang2018revisiting}, and $\dpsgd$ has the best known theoretical guarantees. The $\ssp$ algorithm involves releasing $X^TX$ and $X^T\mathbf{y}$ differentially privately and computing $(\widehat{X^TX})^{-1}\widehat{X^T\mathbf{y}}$. $\dpsgd$ is a private version of SGD where the DP noise is set adaptively according to the excess error in each iteration. For both the algorithms, we use the hyper-parameters recommended in their respective papers. To improve the performance of $\dpsgd,$ we reduce the clipping threshold recommended by the theory by a constant factor. 
\end{itemize}
\textbf{$\dprobgd$.} We implement Algorithm~\ref{alg:main} with the following key changes. Instead of relying on ${\rm Private Norm Estimator}$ to estimate $\Gamma$, we set it to its true value $\Tr(\Sigma).$  This is done for a fair comparison with $\dpsgd$ which assumes the knowledge of $\Tr(\Sigma).$ Next, we use $20\%$ of the samples to compute $\gamma_t$ in line 5 (instead of the $50\%$ stated in Algorithm~\ref{alg:main}). In our experiments we also present results for a variant of our algorithm called $\dprobgdstar$ which outputs the best iterate based on $\gamma_t$, instead of the last iterate. One could also perform tail-averaging instead of picking the best iterate. Both these modifications are primarily used to reduce the variance in the output of Algorithm~\ref{alg:main} and achieved similar performance in our experiments. \vspace{0.1in}\\
 \textbf{Results.}  Figure~\ref{fig:dp_regression}  presents the performance of various algorithms as we vary $n, \kappa, \sigma$. It can be seen that $\dprobgd$ outperforms $\dpsgd$ in almost all the settings (and $\dprobgdstar$ outperforms $\dprobgd$ in all cases). $\ssp$ has poor performance when the noise $\sigma$ is low, but performs slightly better than $\dprobgd$ in other settings. A major drawback of $\ssp$ is its computational complexity which scales as $O(nd^2 + d^{\omega})$. In contrast, the computational complexity of $\dprobgd$ has smaller dependence on $d$ and scales as $\Tilde{O}(nd\kappa).$ Thus the latter is more computationally efficient for high-dimensional problems. 
More experimental results on both robust and private linear regression can be found in the App.~\ref{appendix:experiments}.
\section{Sketch of the main ideas in the analysis}
\label{sec:sketch}

 We provide the main ideas behind the proof of Theorem~\ref{thm:main}. 
The privacy proof is straightforward since no matter what clipping threshold we use the noise we add is always proportionally to the clipping threshold which guarantees privacy. In the remainder, we focus on the utility analysis. 

The proof of the utility heavily relies on the \textit{resilience}~\citep{steinhardt2017resilience} (also known as \textit{stability}~\citep{diakonikolas2019recent}), which states that given a large
enough sample set $S$, various statistics (for example, sample mean and sample variance) of any large enough subset of $S$ will be close to each other. We  define  resilience in App.~\ref{app:res}. 


The main effort for proving Theorem~\ref{thm:main} lies  in the analysis of the gradient descent algorithm. Without clipping and adding noise for differential privacy, convergence property of gradient descent for linear regression is well known. The convergence proof of noisy gradient descent is also relatively straightforward. However, our algorithm requires both clipping and adding noise for robustness and privacy purposes. The key difference between our setting and the classical setting is the existence of adversarial bias and random noise in the gradient. We give an overview of the proof of our robust and private gradient descent as follows. 

First, we introduce some notations. Let ${g}_i^{(t)}:=(x_i^\top w_t-y_i)x_i$ be the raw gradient. Note that when the data follows from our distributional assumption, uncorrupted samples are not clipped: ${\rm clip}_\Theta(x_i)=x_i$ for $i\in S_{\rm good}$. Let $G:=S_{\rm good}\cap S_3=S_3\setminus S_{\rm bad}$ denote the clean data that remains in the input dataset. We can write down one step of gradient update as follows:
\begin{align*}
	&w_{t+1}-w^* 
	\;\;=\;\; w_{t}-\eta\left(\frac{1}{n}\sum_{i\in S}\tilde{g}_i^{(t)}+\phi_t\nu_t\right) -w^*\nonumber\\
	=&\left(\mathbf{I}-\frac{\eta}{n}\sum_{i\in G}x_ix_i^\top\right)(w_{t}-w^*)+\frac{\eta}{n}\sum_{i\in G}x_iz_i+
	\frac{\eta}{n}\sum_{i\in G}(g_i^{(t)}-\tilde{g}_i^{(t)})
	-\frac{\eta}{n}\sum_{i\in S_{\rm bad}}\tilde{g}_i^{(t)} - \eta\phi_t\nu_t \;.
\end{align*}
In the above equation, the first term is a contraction, meaning $w_t$ is moving toward $w^*$. The second term captures the noise from the randomness in the samples. The third term captures the bias introduced by the clipping operation, the fourth term $(\eta/n)\sum_{i\in S_{\rm bad}}\tilde{g}_i^{(t)}$ captures the bias introduced by the adversarial datapoints, and the fifth term captures the added Gaussian noise for privacy. 
The second term is standard and relatively easy to control, and our main focus is on the last three terms. 

The third term
$(\eta/n) \sum_{i\in G}(g_i^{(t)}-\tilde{g}_i^{(t)})$ can be controlled using the resilience property. We prove that with our estimated threshold, the clipping will only affect a small amount of datapoints, whose contribution to the gradient is small collectively. The fourth term $(\eta/n) \sum_{i\in S_{\rm bad}}\tilde{g}_i^{(t)} = ( \eta / n) \sum_{i\in S_{\rm bad}}{\rm clip}_{\theta_t}(x_i^\top w_t-y_i)x_i$ can be controlled since there is only a small amount data points whose label is corrupted, the ${\rm clip}_{\theta_t}(x_i^\top w_t-y_i)$ is controlled by the clipping threshold and the $x_i$ part satisfies resilience property which implies a small, say $S_{\rm bad}$, must have small $\|\sum_{i\in S_{\rm bad}}x_i\|$.

Now we have controlled the deterministic bias. Then, we upper bound the fifth term, which is the noise introduced by the Gaussian noise for the purpose of privacy, and show the expected prediction error decrease in every gradient step. The difficulty is that, since our clipping threshold is adaptive, the decrease of the estimation error depends on the estimation error of all the previous steps. This causes that in some iterations, the estimation error actually increase. In order to get around this, we split the iterations into length $\kappa$ chunks, and argue that the maximum estimation error in a chunk must be a constant factor smaller than the previous chunk. This implies we will reach the desired error within $\tilde{O}(\kappa)$ steps.

\section{Discussion}
We provide a novel variant of DP-SGD algorithm for differentially private linear regression under label corruption. We show the first near-optimal rate that achieves privacy and robustness to label corruptions simultaneously. When there is no label corruption, our result also improves upon the state-of-the-art method \citep{varshney2022nearly} in terms of the condition number $\kappa$. Compared to \citep{varshney2022nearly}, our algorithm has three innovations: $1$) we introduce a novel  adaptive clipping, which is critical in achieving robustness against label corruptions; and $2$) we use full batch gradient descent and a novel convergence analysis to get the near-optimal sample complexity.  Compared to the lower bound and upper bound from a computationally inefficient algorithm in \citep{liu2022differential}, our sample complexities $\tilde{O}(d/\alpha^2+\kappa^{1/2}d/(\varepsilon \alpha))$ has additional $\kappa^{1/2}$ factor in the privacy term. It remains an open question if there is an efficient algorithm to achieve the optimal rate without the $\kappa$ dependence.

\section*{Acknowledgement} 
We thank Abhradeep Guha Thakurta  for helpful discussions while
working on this paper. 
This work is supported in part by  NSF grants  CNS-2002664,  DMS-2134012, CCF-2019844 as a part of NSF Institute for Foundations of Machine Learning (IFML), CNS-2112471 as a part of NSF AI Institute for Future Edge Networks and Distributed Intelligence (AI-EDGE).

\bibliography{ref}
\bibliographystyle{icml2023}

\newpage
\appendix
\onecolumn
\section*{Appendix}

\section{Related work}
\label{app:related} 

{\bf  Differentially private optimization.} There is a long line of work at the intersection of  differentially privacy and optimization \citep{chaudhuri2011differentially, kifer2012private,bassily2014private, song2013stochastic, bassily2019private,wu2017bolt, andrew2021differentially,feldman2020private,song2020characterizing, asi2021private, kulkarni2021private, kamath2021improved,zhang2022bring}. As one of the most well-studied problem in differentially privacy, DP Empirical Risk Minimization (DP-ERM) aims to minimize the empirical risk $(1/n)\sum_{i\in S}\ell(x_i; w)$ privately. The optimal excess empirical risk for approximate DP (i.e., $\delta>0$) is known to be $GD \cdot \sqrt{d}/(\varepsilon n)$, where the loss $\ell$ is convex and  $G$-Lipschitz with respect to the data, and $D$ is the diameter of the convex parameter space \citep{bassily2014private}. This bound can be achieved by several DP-SGD methods, e.g., \citep{song2013stochastic, bassily2014private}, with different computational complexities. Differentially private stochastic convex optimization considers minimizing the population risk $\E_{x\sim \cD}[\ell(x, w)]$, where data is  drawn i.i.d.~from some unknown distribution $\cD$. Using some variations of DP-SGD,  \cite{bassily2019private} and  \cite{feldman2020private} achieves a population risk of $GD (1/\sqrt{n}+ \sqrt{d}/(\varepsilon n))$. 


{\bf DP linear regression.} Applying above results for the linear model, by observing that $G=O(d)$ if $D=O(1)$, the sample complexity required for achieving generalization error is $n=d^2$. Existing works for DP linear regression, for example \citep{vu2009differential,kifer2012private,mir2013differential,dimitrakakis2014robust,wang2015privacy,foulds2016theory, minami2016differential, wang2018revisiting, sheffet2019old,agarwal2019robustness} typically consider deterministic data. Under the i.i.d. Gaussian data setting, this translates into a sample complexity of $n=d^{3/2}/(\varepsilon \alpha)$, where the extra $d^{1/2}$ due to the fact that no statistical assumptions are made. For i.i.d. sub-Weibull data, recent work \citep{varshney2022nearly} achieved nearly optimal excess population risk $d/n+d^2/(\varepsilon^2 n^2)$ using DP-SGD with adaptive clipping, up to extra factors on the condition number. This is closest to our work and we provide detailed comparisons in Sections~\ref{sec:standard} and \ref{sec:analysis}. 
Under Gaussian assumptions, \cite{milionis2022differentially} analyze linear regression algorithm with sub-optimal guarantees. 
\citep{dwork2009differential,amin2022easy,alabi2020differentially, liu2022differential} also consider using robust statistics like Tukey median \citep{tukey1975mathematics} or Theil–Sen estimator \citep{theil1950rank} for differentially private regression. However, \citep{dwork2009differential} and \citep{amin2022easy} lack  utility guarantees and \citep{alabi2020differentially} is restricted to one-dimensional data. \cite{liu2022differential} achieves optimal sample complexity but takes exponential time.

{\bf Robust linear regression.} Robust mean estimation and linear regression have been studied for a long time in the  statistics community \citep{tukey1963less, huber1992robust, tukey1975mathematics}. However, for high dimensional data, these estimators generalizing the notion of median to higher dimensions are typically computationally intractable.  Recent advances in the filter-based algorithms, e.g., \citep{diakonikolas2017being,diakonikolas2020robustly,diakonikolas2019robust,diakonikolas2018list,cheng2019high, dong2019quantum}, achieve nearly optimal guarantees for mean estimation in  time linear in the dimension of the dataset. Motivated by the filter algorithms, \cite{diakonikolas2019efficient,diakonikolas2019sever, prasad2018robust, pensia2020robust, cherapanamjeri2020optimal,jambulapati2020robust} achieved nearly optimal rate with $d$ samples for robust linear regression, where both data $x_i$ and label $y_i$ are corrupted. Another type of efficient methods that achieve similar rates and sample complexity in polynomial time is based on sum-of-square proofs \citep{klivans2018efficient,bakshi2021robust}, which can be computationally expensive in practice.  \cite{zhu2019generalized} and \cite{ liu2022differential} achieves nearly optimal rates using $d$ samples but require exponential time complexities.
An important special case of adversarial corruption is when the adversary only corrupts the response variable in supervised learning \citep{khetan2018learning} and also in unsupervised learning \citep{thekumparampil2018robustness}. 
For linear regression, when there is only label corruptions,  \citep{bhatia2015robust, dalalyan2019outlier, kong2022trimmed} achieve nearly optimal rates with $O(d)$ samples. Under the oblivious label corruption model, i.e., the adversary only corrupts a fraction of labels in complete ignorance of the data, \citep{ bhatia2017consistent, suggala2019adaptive} provide consistent estimator $\hat{w}_n$ such that $\lim _{n \rightarrow \infty} \mathbb{E}\left[\widehat{{w}}_n-{w}^*\right]_2=0$ with $O(d)$ samples.


{\bf Robust and private linear regression.} Under the settings of both DP and data corruptions, the only algorithm by \cite{liu2022differential} achieves nearly optimal rates $\alpha\log(1/\alpha)\sigma$ with optimal sample complexities of $d/\alpha^2+d/(\varepsilon\alpha)$. However, their algorithm requires exponential time complexities.

{\bf Robust and private mean estimation} Based on sum-of-square proofs, recent works \citep{ hopkins2022robustness, alabi2022privately} are able to achieve nearly optimal rates $\alpha\log(1/\alpha)$ with $\tilde{O}(d)$ samples for sub-Gaussian data with known covariance.

\section{Preliminary on differential privacy} 
\label{app:dp}

Our algorithm builds upon two DP primitive: Gaussian mechanism and private histogram. The Gaussian mechanism is one examples of a larger family of mechanisms known as output perturbation mechanisms. 
In practice, it is possible to get better utility trade-off for a output perturbation mechanism by carefully designing the noise, such as the stair-case mechanism which are shown to achieve optimal utility in the variance \citep{geng2015staircase} and also in hypothesis testing \citep{kairouz2014extremal}. However, the gain is only by constant factors, which we do not try to optimize in this paper. We provide a reference for the Gaussian mechanism and private histogram below.

\begin{lemma} [{Gaussian mechanism \citep{dwork2014algorithmic}}]
\label{lem:gauss} 
    For a query $q$ with sensitivity $\Delta_q$, the Gaussian mechanism outputs $q(S)+\cN(0,(\Delta_q\sqrt{2\log(1.25/\delta)}/\varepsilon)^2 {\bf I}_d)$ and achieves $(\varepsilon,\delta)$-DP. 
\end{lemma} 
\begin{lemma}[Stability-based histogram {\citep[Lemma~2.3]{karwa2018finite}}]\label{lem:hist-KV17} For every $K\in \mathbb{N}\cup \infty$, domain $\Omega$, for every collection of disjoint bins $B_1,\ldots, B_K$ defined on $\Omega$, $n\in \mathbb{N}$, $\eps\geq 0,\delta\in(0,1/n)$, $\beta>0$ and $\alpha\in (0,1)$ there exists an $(\eps,\delta)$-differentially private algorithm $M:\Omega^n\to \mathbb{R}^K$ such that for any set of data $X_1,\ldots,X_n\in \Omega^n$
\begin{enumerate}
\item $\hat{p}_k = \frac{1}{n}\sum_{X_i\in B_k}1$
\item $(\tilde{p}_1,\ldots,\tilde{p}_K)\gets M(X_1,\ldots,X_n),$ and
\item
$$
n\ge \min\left\{\frac{8}{\eps\beta}\log(2K/\alpha),\frac{8}{\eps\beta}\log(4/\alpha\delta)\right\} 
$$
\end{enumerate}
then,
$$
\mathbb{P}(|\tilde{p}_k-\hat{p}_k|\le\beta)\ge 1-\alpha
$$
\end{lemma}

When the databse is accessed multiple times, we use the following composition theorems to account for the end-to-end privacy leakage. 

\begin{lemma}[Parallel composition \cite{mcsherry2009privacy}]
    \label{lem:parallel} Consider a sequence of interactive queries 
    $\{q_k\}_{k=1}^K$ each operating on a subset $S_k$ of the database and each satisfying  $(\varepsilon, \delta)$-DP.
    If  $S_k$'s are disjoint then the composition $(q_1(S_1), q_2(S_2), \ldots, q_K(S_K))$ is
    $(\varepsilon,\delta)$-DP.
\end{lemma}

\begin{lemma}[Serial composition \cite{dwork2014algorithmic}]
    \label{lem:serial} 
    If a database is accessed with an $(\varepsilon_1,\delta_1)$-DP mechanism and then with an $(\varepsilon_2,\delta_2)$-DP mechanism, then the end-to-end privacy guarantee is $(\varepsilon_1+\varepsilon_2,\delta_1+\delta_2)$-DP.
\end{lemma}
In most modern privacy analysis of iterative processes,  advanced composition theorem from \cite{kairouz2015composition} gives tight accountant for the end-to-end privacy budget. It can be improved for specific mechanisms using tighter accountants, e.g., in \cite{mironov2017renyi,girgis2021renyi,wang2019subsampled,zhu2022optimal,gopi2021numerical}. 
\begin{lemma}[Advanced composition \cite{kairouz2015composition}]
    \label{lem:composition}
    For $\varepsilon\leq0.9$, 
    an end-to-end guarantee of $(\varepsilon,\delta)$-differential privacy is satisfied if a database  is accessed $k$ times, each with  a $(\varepsilon/(2\sqrt{2k\log(2/\delta)}),\delta/(2k))$-differential private mechanism. 
\end{lemma}

\section{Definition of resilience}
\label{app:res}

\begin{definition}[{\citep[Definition~23]{liu2022differential}}]
    \label{def:res}
    For some $\alpha\in (0,1)$, $\rho_1\in \reals_+$,  $\rho_2\in  \reals_+$, and $\rho_3\in  \reals_+$, $\rho_4\in \reals_+$, we say dataset $S_{\rm good}=\{(x_i\in\reals^d,y_i\in\reals)\}_{i=1}^n$ is $(\alpha,\rho_1,\rho_2, \rho_3, \rho_4)$-resilient with respect to $(w^*,\Sigma,\sigma) $ for some $w^*\in\reals^d$, positive definite $\Sigma\succ0\in\reals^{d\times d}$, and $\sigma>0$ if for any $T\subset S_{\rm good}$ of size $|T|\geq (1-\alpha)n$, the following holds for all $v\in \reals^d$:     
    \begin{align}
 &\Big|\frac{1}{|T|}\sum_{(x_i,y_i)\in T} \langle v,x_i\rangle (y_i-x_i^\top w^*)\Big| \leq  \rho_1 \,\sqrt{v^\top\Sigma v} \, \sigma\; \text{ , } \label{def:res1}\\
	&\Big| \frac{1}{|T|}\sum_{x_i\in T} \langle v, x_i\rangle^2  - v^\top\Sigma v \Big|  \leq   \rho_2  v^\top\Sigma v\;, \label{def:res2} \\
	&\Big| \frac{1}{|T|}\sum_{(x_i, y_i)\in T} (y_i-x_i^\top w^*)^2-\sigma^2  \Big|   \leq   \rho_3  \sigma^2 \;,   \label{def:res3}\\
	&\Big| \frac{1}{|T|}\sum_{(x_i, y_i)\in T} \ip{v}{x_i}  \Big|   \leq   \rho_4  \sqrt{v^\top\Sigma v}    \label{def:res4}\;.
    \end{align}
\end{definition}

\section{Proof of Theorem~\ref{thm:distance} on the private distance estimation}

\label{app:proof_distance}

	We first analyze the privacy. Changing a data point $(x_i, y_i)$ can affect at most one partition in $\{G_j\}_{j=1}^k$. This would affect at most two histogram bins, increasing the count of one bin by one and decreasing the count in another bin by one. Under such a bounded $\ell_1$ sensitivity, the privacy guarantees follows from Lemma~\ref{lem:hist-KV17}.
	
	Next, we analyze the utility. In the (private) histogram step, we claim that at most only two consecutive bins can be occupied by any $\phi_j$'s. This is also true for the private histogram, because the private histogram 
	of Lemma~\ref{lem:hist-KV17} adds noise to non-empty bins only.  
	By   Lemma~\ref{lem:hist-KV17}, if $k\geq c \log(1/(\delta_0\zeta_0))/\varepsilon_0$, one of these two intervals (the union of which contains the true distance $\|w_t-w^*\|_\Sigma^2 + \sigma^2$) is released. This results in a multiplicative error bound of four, as the bin size increments by a factor of two.
		
	To show that only two bins are occupied, we show that all $\phi_j$'s are close to the true distance.  
	We first show that each partition contains at most $2\alpha_{\rm corrupt}$ fraction of corrupted samples  and thus all partitions are   $(2\bar\alpha, 6\bar{\alpha}, 6\hat\rho,  6\hat\rho,  6\hat\rho,6\hat\rho')$-\text{corrupt good},
	where $ \hat\rho(C_2,K,a,\bar\alpha)= C_2K^2\bar{\alpha}\log^{2a} (1/6\bar{\alpha})$ and $ \hat\rho'(C_2,K,a,\bar\alpha)= C_2K\bar{\alpha}\log^{a} (1/6\bar{\alpha})$, as defined in Definition~\ref{def:corruptgood}.
	
	Let $B=\lfloor n/k\rfloor$ be the sample size in each partition. Let $\zeta_0=\zeta/2$. Since the partition is drawn uniformly at random, for each partition $G_j$, the number of corrupted samples   $\alpha' n$ satisfies $\alpha' n\sim {\rm Hypergeometric}(n, \alpha_{\rm corrupt} n, n/k)$. The tail bound gives that with probability $1-\zeta_0$, 
	\begin{align*}
		\alpha'\leq \alpha_{\rm corrupt}+(k/n)  \log(2/\zeta_0)\leq 2\bar\alpha \;,
	\end{align*}
	where the last inequality follows from the fact that the corruption level is bounded by $\alpha_{\rm corruption}\leq\bar\alpha $ and the assumption on the sample size in Eq.~(\ref{eq:distance}) which implies $n\gtrsim \log(1/(\delta_0\zeta_0))\log(1/\zeta_0)/(\bar\alpha \varepsilon_0)$.

	For a particular subset $G_j$, Lemma~\ref{lemma:subweibull_res_conditions} implies that if $B=O((d+\log(1/\zeta_0))/\bar{\alpha}^2)$, then $G_j$ is $(\alpha', 6\bar{\alpha},6\hat\rho, 6\hat\rho, 6\hat\rho, 6\hat\rho')$-corrupt good set with respect to $(w^*, \Sigma, \sigma)$ from Assumption~\ref{asmp:distribution}. 
	This means that there exists a constant $C_2>0$ such that for any $T_1\subset S_{\rm good}$ with $|T_1|\geq (1-6\bar{\alpha})B$, we have 
	\begin{align*}
		\left|\frac{1}{|T_1|}\sum_{i\in T_1 }\ip{x_i}{w^*-w_t}^2-\|w^*-w_t\|_{\Sigma}^2\right|\leq 6C_2K^2\bar{\alpha}\log^{2a}(1/(6\bar{\alpha}))\|w^*-w_t\|_{\Sigma}^2\;,
	\end{align*} 
	\begin{align*}
		\left|\frac{1}{|T_1|}\sum_{i\in T_1}z_i^2-\sigma^2\right|\leq 6C_2K^2\bar{\alpha}\log^{2a}(1/(6\bar{\alpha}))\sigma^2\;,
	\end{align*}
	and
	\begin{align*}
		\left|\frac{1}{|T_1|}\sum_{i\in T_1}z_i\ip{x_i}{w^*-w_t}\right|\leq 6C_2K^2\bar{\alpha}\log^{2a}(1/(6\bar{\alpha}))\|w^*-w_t\|_{\Sigma}\sigma\;.
	\end{align*}
	
	Note that for $i\in S_{\rm good}$, $b_i = z_i^2+2z_i(w^*-w_t)^\top x_i+(w^*-w_t)^\top x_ix_i^\top (w^*-w_t)$. By the triangular inequality, we know, under above conditions,
	\begin{align}
		\left|\frac{1}{|T_1|}\sum_{i\in T_1}b_i-\|w^*-w_t\|_{\Sigma}^2-\sigma^2\right|\leq 12C_2K^2\bar{\alpha}\log^{2a}(1/(6\bar{\alpha}))(\|w^*-w_t\|_{\Sigma}^2+\sigma^2)\;.\label{eq:res_a_i}
	\end{align}
	Which also implies that any subset $T_2\subset S_{\rm good}$ and $|T_2|\leq 6\bar\alpha |S_{\rm good}|$, we have
	\begin{align}
		\left|\frac{1}{|T_2|}\sum_{i\in T_2}b_i-\|w^*-w_t\|_{\Sigma}^2-\sigma^2\right|\leq 12C_2K^2\log^{2a}(1/(6\bar{\alpha}))(\|w^*-w_t\|_{\Sigma}^2+\sigma^2)\;.
		\label{eq:res_small_set}
	\end{align}
	 Recall that $\psi_j$ is the $(1-3\bar\alpha)$-quantile of the dataset $G_j$. Let $T:=\{i\in  S_{\rm good}: b_i\leq \psi_j\}$,  where with a slight abuse of notations, we use $S_{\rm good}$ to denote the set of uncorrupted samples corresponding to $G_j$ and $S_{\rm bad}$ to denote the set of corrupted samples corresponding to $G_j$. Since the corruption is less than $\alpha'$,  we know $(1-3\bar{\alpha}-\alpha')B\leq |T|\leq (1-3\bar{\alpha}+\alpha')B$. By our assumption that  $\alpha'\leq  2\bar{\alpha}$,  we have  $|\bar{E}|\geq (3\bar{\alpha}-\alpha')B\geq \bar{\alpha}B$ where  $\bar{E}:=S_{\rm good}\setminus E$. 
	 Using \Eqref{eq:res_small_set} with a choice of $T_2 = \bar E$, we get that 
	 \begin{align}
		\min_{i\in \bar{E}}b_i-\|w^*-w_t\|_{\Sigma}^2-\sigma^2 \leq 12C_2K^2\log^{2a}(1/(6\bar{\alpha}))(\|w^*-w_t\|_{\Sigma}^2+\sigma^2)\;.
	\end{align}
	This implies that 
	\begin{align}
	\psi_j \le 12C_2K^2\log^{2a}(1/(6\bar{\alpha}))(\|w^*-w_t\|_{\Sigma}^2+\sigma^2) \label{eq:quantile_bound}.
	\end{align}
	Hence 
	\begin{align}
		&\left|\phi_j - \|w^*-w_t\|_{\Sigma}^2-\sigma^2\right|  = \left|\frac{1}{B}\sum_{i\in G_j }b_i\cdot \mathbf{1}\{b_i\le \psi_j\} - \|w^*-w_t\|_{\Sigma}^2-\sigma^2\right|\nonumber\\
		&\;\;\;\;\;=\left|\frac{1}{B}\sum_{i\in T}b_i - \|w^*-w_t\|_{\Sigma}^2-\sigma^2\right| + \left|\frac{1}{B}\sum_{i\in S_{\rm bad} }b_i \cdot \mathbf{1}\{b_i\le \psi_j\} \right|\nonumber \\
		&\;\;\;\;\;\leq 37C_2K^2\cdot \bar{\alpha}\log^{2a}(1/(6\bar{\alpha}))(\|w^*-w_t\|_{\Sigma}^2+\sigma^2)\label{eq:res_upper},
	\end{align}
	where we applied \Eqref{eq:quantile_bound} and \Eqref{eq:res_a_i} in the last inequality.

	On a fixed partition $G_j$, we showed that if $B=O((d+\log(1/\zeta_0))/\bar{\alpha}^2) $ then, with probability $1-\zeta_0$,  $|\phi_j - \|w^*-w_t\|_{\Sigma}^2 - \sigma^2| \leq \frac{1}{4}(\|w^*-w_t\|_{\Sigma}^2+\sigma^2)$, which follows from our assumption that $37C_2K^2\cdot \bar{\alpha}\log^{2a}(1/(6\bar{\alpha}))\le 1/4$. Using an union bound for all subsets, we know if $B=O((d+\log(k/\zeta_0))/\bar{\alpha}^2)$, then $1-\zeta_0$, $|\phi_j - \|w^*-w_t\|_{\Sigma}^2 - \sigma^2|\leq \frac{1}{4}(\|w^*-w_t\|_{\Sigma}^2+\sigma^2)$ holds for all $j\in [k]$. Since the upper bound lower bound ratio is $5/3$ which is less than $2$. All the $\phi_j$ must lie in two bins, which will result in a factor of $4$ multiplicative error.
	
	\section{Proof of  Lemma~\ref{lemma:clipping_fraction} on the upper bound on clipped good points}
	\label{app:proof_clipping_fraction} 
	
 Let $\hat\rho(C_2,K,a,\alpha)= 2C_2K^2\alpha\log^{2a} (1/(2\alpha))$ and $ \hat\rho'(C_2,K,a,\alpha)= 2C_2K\alpha\log^{a} (1/(2\alpha))$. Lemma~\ref{lemma:subweibull_res_conditions} implies that if $n=O((d+\log(1/\zeta))/(\alpha^2))$ with a large enough constant, then there exists a universal constant $C_2$ such that $S_3$ is, with respect to $(w^*,\Sigma,\sigma)$,  $(\alpha_{\rm corrupt}, 2\alpha, \hat\rho,\hat\rho,\hat\rho, \hat\rho')$-corrupt good. The rest of the proof is under this (deterministic) resilience condition. By the resilience property in \Eqref{def:res2}, we know for any $T\subset S_{\rm good}$ with $|T|\geq(1-2\alpha)n$, 
	\begin{align}
	\left|\frac{1}{|T|}\sum_{i\in T}(w^*- w_t)^\top x_ix_i^\top (w^*- w_t)-\|w^*-w_t\|_\Sigma^2\right|
	&\leq  2C_2K^2\alpha\log^{2a}(1/(2\alpha)) \|w^*-w_t\|_\Sigma^2\;.\label{eq:beta_t_res1}
	\end{align} 
	
	Let $E:= \left\{i\in S_{\rm good}:(w^*- w_t)^\top x_ix_i^\top (w^*- w_t)> \|w^*-w_t\|_\Sigma^2(8C_2K^2\log^{2a}(1/(2\alpha))+1)\right\}$. Denote $\tilde{\alpha}:=|E|/n$. We want to show that $\tilde{\alpha}\leq \alpha/2$. Let $T$ be the set of points that contain the smallest $1-\alpha/2$ fraction in $\{(w^*- w_t)^\top x_ix_i^\top (w^*- w_t)\}_{i\in S_{\rm good}}$. We know $|T|=(1-\alpha/2)n\geq (1-2\alpha)n$. To prove by contradiction, suppose $\tilde{\alpha}>\alpha/2$, which means all data points in $S_{\rm good}\setminus T$ are larger than $\|w^*-w_t\|_\Sigma^2(8C_2K^2\log^{2a}(1/(2\alpha))+1)$. From resilience property in \Eqref{eq:beta_t_res1}, we know 
	\begin{align*}
		&\frac{1}{n}\sum_{i\in S_{\rm good}}(w^*- w_t)^\top x_ix_i^\top (w^*- w_t)\\
		=\;& \frac{1}{n}\sum_{i\in T}(w^*- w_t)^\top x_ix_i^\top (w^*- w_t)+\frac{1}{n}\sum_{i\in S_{\rm good}\setminus T}(w^*- w_t)^\top x_ix_i^\top (w^*- w_t)\\
		\geq\; & \Big(1-\frac\alpha2\Big) \left(1-2C_2K^2\alpha\log^{2a}(\frac{1}{2\alpha}) \right)\|w^*-w_t\|_\Sigma^2 + \frac\alpha2 \, (8C_2K^2\log^{2a}(\frac{1}{2\alpha})+1)\|w^*-w_t\|_\Sigma^2\\
		> \;& (1+2C_2K^2\alpha\log^{2a}(1/{2\alpha}))\|w^*-w_t\|_\Sigma^2\;,
	\end{align*}
	which contradicts \Eqref{eq:beta_t_res1} for $S_{\rm good}$. This shows $\tilde{\alpha}\leq \alpha/2$. 
	
	Similarly, we can show that $\left|\left\{i\in S_{\rm good}:z_t^2> \sigma^2(8C_2K^2\log^{2a}(1/(2\alpha))+1)\right\}\right|\leq \alpha/2$. This means the rest $(1-\alpha)n$ points in $S_{\rm good}$ satisfies $\sqrt{(w^*- w_t)^\top x_ix_i^\top (w^*- w_t)} +|z_i| \leq (\|w_t-w^*\|+\sigma)\sqrt{(8C_2K^2\log^{2a}(1/(2\alpha))+1)}$. 
 Note that for all $i\in S_{\rm good}$, we have
	\begin{align*}
	|x_i^\top w_t-y_i| &=  \left| x_i^\top(w_t- w^*)-z_i\right|\\
	&\leq |x_i^\top(w_t- w^*)|+|z_i|\\
	&\leq \left(\sqrt{(w^*- w_t)^\top x_ix_i^\top (w^*- w_t)} +|z_i|\right)\;.
\end{align*}
By our assumption that $C_2K^2\log^{2a}(1/(2\bar\alpha))\geq 1$ which follows from Assumption~\ref{asmp:corrupt}, we have 
\begin{align}
	\left|\left\{i\in S_{\rm good}: \|x_i^\top w_t-y_i\|\leq (\|w_t-w^*\|+\sigma)\sqrt{9C_2K^2\log^{2a}(1/(2\alpha))}\right\}\right|\geq (1-\alpha) n\;.
\end{align}

\section{Private norm estimation: algorithm and analysis} 
\label{app:norm}

\begin{algorithm2e}[H]    
   \caption{Private Norm Estimator} 
   \label{alg:norm} 
   	\DontPrintSemicolon 
	\KwIn{$S_1=\{(x_i, y_i)\}_{i=1}^n$, target privacy $(\varepsilon_0,\delta_0)$, failure probability $\zeta$.}
	\SetKwProg{Fn}{}{:}{}
	{ 
	Let $a_i\gets \|x_i\|^2$. Let $\tilde{S}=\{a_i\}_{i=1}^n$.\\
	Partition $\tilde{S}$ into $k=\lfloor C_1\log(1/(\delta_0\zeta))/\varepsilon\rfloor$ subsets of equal size and let $G_j$ be the $j$-th partition.\\
	
	For each $j\in [k]$, denote $\psi_j= (1/|G_j|) \sum_{i\in G_j}a_i$.\\	
	Partition $[0, \infty)$ into bins of geometrically increasing intervals $\Omega:= \left\{\ldots,\left[ 2^{-2/4}, 2^{-1 / 4}\right),\left[ 2^{-1 / 4}, 1\right),\left[1,2^{1 / 4}\right),\left[2^{1 / 4}, 2^{2/4}\right), \ldots\right\} \cup\{[0,0]\}$\\
	Run $(\varepsilon_0, \delta_0)$-DP histogram learner of Lemma~\ref{lem:hist-KV17} on $\{ \psi_j\}_{j=1}^k$ over $\Omega$ \\
	{\bf if} all the bins are empty {\bf then} 
	Return $\perp$\\
	Let $[\ell, r]$ be a non-empty bin that contains the maximum number of points in the DP histogram\\
	Return  $\ell$
	} 
\end{algorithm2e}

\begin{lemma}
\label{lem:norm}
	Algorithm~\ref{alg:norm} is $(\varepsilon_0, \delta_0)$-DP. If $\{x_i\}_{i=1}^n$ are i.i.d.~samples from $(K,a)$-sub-Weibull distributions with zero mean and covariance $\Sigma$ and 
	\begin{align*}
		n=\tilde{O}\left(\frac{\log^{2a}(1/(\delta_0\zeta))}{\varepsilon_0}\right)\;,
	\end{align*}
	with a large enough constant then 
Algorithm~\ref{alg:norm} returns $\Gamma$ such that, with probability $1-\zeta$, 
\begin{align*}
	\frac{1}{\sqrt{2}}\Tr(\Sigma)\leq \Gamma\leq \sqrt{2}\Tr(\Sigma)\;.
\end{align*}
\end{lemma}
We provide a proof in App.~\ref{app:proof_norm}.

\subsection{Proof of Lemma~\ref{lem:norm} on the private norm estimation}
\label{app:proof_norm}

By Hanson-Wright inequality in Lemma~\ref{lemma:hanson} and union bound, there exists constant $c>0$ such that with probability $1-\zeta$,
\begin{align}
	|\frac{1}{b}\sum_{i=1}^b \|x_i\|^2-\Tr(\Sigma)|\leq cK^2\Tr(\Sigma)\left(\sqrt{\frac{\log(1/\zeta)}{b}}+\frac{\log^{2a}(1/\zeta)}{b}\right)\;,
\end{align}

This means there exists a constant $c'>0$ such that if $b\geq c'K^2\log^{2a}(k/\zeta)$, then for all $j\in [k]$.
\begin{align}
	|\psi_j-\Tr(\Sigma)|\leq 2^{1/8}\Tr(\Sigma)
\end{align} 

With probability $1-\zeta$, $\{\psi_j\}_{j=1}^k$  lie in interval of size $2^{1/4}\Tr(\Sigma)$. Thus, at most two consecutive bins are filled with $\{\psi_j\}_{j=1}^k$. Denote them as $I=I_1\cup I_2$.  Our analysis indicates that $\prob(\psi_i\in I)\geq 0.99$. By private histogram in Lemma~\ref{lem:hist-KV17}, if $k\geq \log(1/(\delta\zeta))/\varepsilon$, $|\hat{p}_I-\tilde{p}_I|\leq 0.01$ where $\hat{p}_I$ is the empirical count on $I$ and $\tilde{p}_I$ is the noisy count on $I$. Under this condition, one of these two intervals are released. This results in multiplicative error of $\sqrt{2}$. 	

\section{Proof of the resilience  in  Lemma~\ref{lemma:subweibull_res_conditions}}
\label{app:proof_subweibull_res_conditions}

We apply following resilience property for general distribution characterized by Orlicz function from \cite{zhu2019generalized}. 

 \begin{lemma}[{\citep[Theorem~3.4]{zhu2019generalized}}]
 \label{lemma:res_orlicz}
 	Dataset $S=\{x_i\in \reals^d\}_{i=1}^n$ consists i.i.d. samples from a distribution $\cD$.  Suppose $\cD$ is zero mean and satisfies $\E_{x\sim \cD}\left[\psi\left(\frac{(v^\top x)^2}{\kappa^2\E_{x\sim \cD}[(v^\top x)^2]}\right)\right]\leq 1$ for all $v\in \reals^d$, where $\psi(\cdot)$ is Orlicz function.  Let $\Sigma=\E_{x\sim \cD}[xx^\top]$. Suppose $\alpha\leq \bar{\alpha}$, where $\bar{\alpha}$ satisfies $(1+\bar{\alpha}/2)\cdot 2\kappa^2\bar{\alpha}\psi^{-1}(2/\bar{\alpha})<1/3$, $\bar{\alpha}\leq 1/4$. Then there exists constant $c_1, C_2$ such that if  $n\geq c_1((d+\log(1/\zeta)) / (\alpha^2))$, with probability $1-\zeta$, for any $T\subset S$ of size $|T|\geq(1-\alpha)n$, the following holds:
	 \begin{align}
	 	\left\|\Sigma^{-1/2}\left(\frac{1}{|T|}\sum_{i\in T}x_i\right)\right\| & \leq   C_2\kappa\alpha\sqrt{\psi^{-1}(1/\alpha)} \label{eq:res_or1}
	 \end{align}
and	 
	\begin{align}
	 	\left\|\mathbf{I}_d-\Sigma^{-1/2}\left(\frac{1}{|T|}\sum_{i\in T}x_ix_i^\top \right)\Sigma^{-1/2}\right\|_2 & \leq   C_2\kappa^2\alpha\psi^{-1}(1/\alpha)\;.\label{eq:res_or2}
	 \end{align}
 \end{lemma} 

Let $\psi(t) = e^{t^{1/(2a)}}$. It is easy to see that $\psi(t)$ is a valid Orlicz function. Then if $x_i$ is $(K, a)$-sub-Weibull, then we know 
\begin{align}
	\left\|\Sigma^{-1/2}\left(\frac{1}{|T|}\sum_{i\in T}x_i\right)\right\| & \leq   C_2K\alpha\sqrt{\log^{2a}(1/\alpha)} \label{eq:res_tail1}\;,
\end{align}
and
\begin{align}
	\left\|\mathbf{I}_d-\Sigma^{-1/2}\left(\frac{1}{|T|}\sum_{i\in T}x_ix_i^\top \right)\Sigma^{-1/2}\right\|_2 & \leq   C_2K^2\alpha\log^{2a}(1/\alpha)\;.
\end{align}

This implies
\begin{align}
	(1-C_2K^2\alpha\log^{2a}(1/\alpha))\mathbf{I}_d\preceq \Sigma^{-1/2}\left(\frac{1}{|T|}\sum_{i\in T}x_ix_i^\top \right)\Sigma^{-1/2} \preceq (1+C_2K^2\alpha\log^{2a}(1/\alpha))\mathbf{I}_d\;.
\end{align}
Using the fact that $C^\top AC\preceq C^\top BC$ if $A\preceq B$, we know
\begin{align}
(1-C_2K^2\alpha\log^{2a}(1/\alpha))\Sigma	\preceq \frac{1}{|T|}\sum_{i\in T}x_ix_i^\top \preceq (1+C_2K^2\alpha\log^{2a}(1/\alpha))\Sigma\;.\label{eq:res_tail2}
\end{align}

This implies resilience properties of $x_i$ and $z_i$ in \eqref{def:res2} and \eqref{def:res3} in Definition~\ref{def:res} respectively. Next, we show the resilience property of $x_iz_i$.

By  $ab\leq \frac{a^2}{2}+\frac{b^2}{2}$, for any fixed $v\in \reals^d$, 
\begin{align}
	\E[\exp \left(\left(\frac{|\ip{x_i z_i}{v}|^{2}}{K^4\sigma^2v^\top\Sigma v}\right)^{1 /(4a)}\right)] &\leq  \E\left[\exp \left(\left(\frac{|\ip{x_i }{v}|^{2}}{K^{2}v^\top \Sigma v }\right)^{1 / (2 a)}/2\right) \exp \left(\left(\frac{z_i^{2}}{K^{2}\sigma^2}\right)^{1 / (2 a)}/2\right)\right]\\
	&\leq \frac{1}{2} \left(\E\left[\exp \left(\left(\frac{|\ip{x_i }{v}|^{2}}{K^{2}v^\top\Sigma v}\right)^{1 / (2 a)}\right)\right]+\E\left[\exp \left(\left(\frac{z_i^{2}}{K^{2}\sigma^2}\right)^{1 / (2 a)}\right)\right]\right)\\
	&\leq 2\;.
\end{align}

Since $\E[x_iz_i]=0$, {\citep[Lemma~E.3]{zhu2019generalized}} implies that there exists constant $c_1,C_2>0$ such that if $n\geq c_1(d+\log(1/\zeta))/(\alpha^2)$, with probability $1-\zeta$, for any $T\subset S_{\rm good}$ of size $|T|\geq (1-\alpha)n$, 
\begin{align}
	\left\|\Sigma^{-1}\left(\frac{1}{|T|}\sum_{i\in T}x_iz_i\right)\right\|\leq C_2K^2\sigma\alpha\log^{2a}(1/\alpha)\;.
\end{align}

%
%


%
%
%
%


\section{Proof of Theorem~\ref{thm:main} on the analysis of Algorithm~\ref{alg:main}}
\label{app:proof_main}

The main theorem builds upon the following lemma   that analyzes a (stochastic) gradient descent method, where the randomness is from the DP noise we add and the analysis only relies on certain deterministic conditions on the dataset including resilienece and concentration. Theorem~\ref{thm:main} follows in a straightforward manner by collecting Theorem~\ref{thm:distance}, Lemma~\ref{lem:norm},  Lemma~\ref{lemma:clipping_fraction}, and Lemma~\ref{lem:gd}.

\begin{lemma} 
\label{lem:gd} 
Algorithm~\ref{alg:main} is $(\varepsilon, \delta)$-DP. 
Under Assumptions~\ref{asmp:distribution} and \ref{asmp:corrupt} for any $\zeta\in(0,1)$ and $\alpha\geq \alpha_{\rm corrupt}$ satisfying $K^2\alpha\log^{2a}(1/\alpha)\log(\kappa)\leq c$ for some universal constant $c>0$, if distance threshold is small enough such that 
	\begin{align}
	\theta_t \;\; \leq \;\; {3C_2^{1/2} K \log^{a}(1/(2\alpha ))} \cdot \left(\|w^*-w_t\|_\Sigma+\sigma\right)\;, 
\end{align} 
and large enough such that the number of clipped clean data points is no larger than $\alpha  n $, at every round, the norm threshold is large enough such that 
\begin{align}
    \Theta  \;\; \geq \;\; K\sqrt{\Tr(\Sigma)}\log^{a}(n/\zeta)\;,
\end{align}
and  sample size is large enough such that 
	\begin{align}
	    n = O \left(K^2d\log(d/\zeta)\log^{2a}(n/\zeta)+\frac{d+\log(1/\zeta)}{\alpha^2}+
	    \frac{K^2 T^{1/2} d \log(T/\delta) \log^{a}(n/(\alpha \zeta))}{\varepsilon \alpha}\right)\;,
	\end{align} with a large enough constant, then the choices of a  step size, $\eta =  1/(C\lambda_{\max}(\Sigma))$ for some $C\geq 1.1$, and the number of iterations,  
	    $T= \tilde\Theta\left(\kappa \log\left(\|w^*\|
	    \right)\right)\,$, 
	ensures that  Algorithm~\ref{alg:main} outputs $w_T$ satisfying the following with probability $1-\zeta$: 
	\begin{align}
		 &\E_{\nu_1, \cdots, \nu_t\sim \cN(0, \mathbf{I}_d)}[\| w_T - w^*\|_\Sigma^2] \;\;
		 \lesssim  \;\; K^4\sigma^2 \log^2(\kappa) \alpha^2\log^{4a}(1/\alpha)\;,
		 \label{eq:rhs} 
	\end{align}
	where the expectation is taken over the noise added for DP and $\tilde\Theta(\cdot)$ hides logarithmic terms in $K,\sigma,d,n,1/\varepsilon,\log(1/\delta),1/\alpha$. 
\end{lemma}

\begin{proof}[Proof of Lemma~\ref{lem:gd}] We first prove a set of deterministic conditions on the clean dataset, which is sufficient for the analysis of the gradient descent.

{\bf Step 1: Sufficient deterministic conditions on the clean dataset.} 
Let $S_{\rm good}$ be the uncorrupted dataset for $S_3$ and $S_{\rm bad}$ be the corrupted datapoints in $S_3$.
Let $G:=S_{\rm good}\cap S_3=S_3\setminus S_{\rm bad}$ denote the clean data that remains in the input dataset. Let $\lambda_{\rm max}=\|\Sigma\|_2$. 
Define $\hat{\Sigma}:=(1/n) \sum_{i\in G}x_ix_i^\top$, $\hat{B}:=\mathbf{I}_d-\eta\hat{\Sigma}$. Lemma~\ref{lemma:cov_concentration_tail} implies that if $n=O(K^2d\log(d/\zeta)\log^{2a}(n/\zeta))$, then
\begin{align}
	0.9\Sigma\preceq \hat{\Sigma}\preceq 1.1\Sigma\;. \label{eq:cov_asmp}
\end{align}
We pick step size $\eta$ such that $\eta\leq 1/(1.1\lambda_{\max})$ to ensure that  $\eta\leq 1/\|\hat{\Sigma}\|_2$. 
Since the covariates $\{x_i\}_{i\in S}$ are not corrupted, from Lemma~\ref{lemma:norm_a_tail}, we know with probability $1-\zeta$, for all $i\in S_3$, 
\begin{align}
	\|x_i\|^2\leq K^2\Tr(\Sigma)\log^{2a}(n/\zeta)\label{eq:norm_asmp}\;.
\end{align}
 Lemma~\ref{lemma:subweibull_res_conditions} implies that if $n=O((d+\log(1/\zeta))/(\alpha^2))$, then there exists a universal constant $C_2$ such that $S_3$ is, following Definition~\ref{def:corruptgood}, with respect to $(w^*,\Sigma,\sigma)$,  \\$(\alpha_{\rm corrupt}, \alpha, C_2K^2\alpha\log^{2a}(1/\alpha),C_2K^2\alpha\log^{2a}(1/\alpha),C_2K^2\alpha\log^{2a}(1/\alpha), C_2K\alpha\log^{a}(1/\alpha))$-corrupt good. Such corrupt good sets have a sufficiently large, $1-\alpha_{\rm corrupt}$, fraction of points that satisfy a good property that we need: resilience. The rest of the proof is under \Eqref{eq:cov_asmp}, \Eqref{eq:norm_asmp}, and that $S_{\rm good}$ is resilient.

{\bf Step 2: Upper bounding the deterministic noise in the gradient.} In this step, we bound the deviation of the gradient from its mean. There are several sources of deviation: $(i)$ clipping, $(ii)$  adversarial corruptions, and $(iii)$ randomness of the data noise and privacy noise. We will show that deviations from all these sources can be controlled deterministically under the corrupt-goodness (i.e., resilience). 

Let $\phi_t=(\sqrt{2\log(1.25/\delta_0)}\Theta\theta_t)/(\varepsilon_0 n)$, which ensures that we add enough noise to guarantee $(\varepsilon_0,\delta_0)$-DP for each step of gradient descent. This follows from the standard Gaussian mechanism in Lemma~\ref{lem:gauss} and the fact that each gradient is clipped to the norm of $\Theta\theta_t$, resulting in a DP sensitivity of $\Theta\theta_t/n$. The fact that this sensitivity scales as $1/n$ is one of the main reasons for the performance gain we get over \cite{varshney2022nearly} that uses a minimatch of size $n/\kappa$ with sensitivity scaling as $\kappa/n$.  
Define $g_i^{(t)}:=x_i(x_i^\top w_t-y_i)$. For $i\in S_{\rm good}$, we know $y_i=x_i^\top w^*+z_i$. Let $\tilde{g}_i^{(t)}={\rm clip}_{\Theta}(x_i){\rm clip}_{\theta_t}(x_i^\top w_t-y_i)$. Note that under \Eqref{eq:norm_asmp}, ${\rm clip}_\Theta(x_i)=x_i$ for all $i\in S_3$. 

From Algorithm~\ref{alg:main}, we can write one-step update rule as follows:
\begin{align}
	&w_{t+1}-w^* \nonumber \\
	=& w_{t}-\eta\left(\frac{1}{n}\sum_{i\in S}\tilde{g}_i^{(t)}+\phi_t\nu_t\right) -w^*\nonumber\\
	=&\left(\mathbf{I}-\frac{\eta}{n}\sum_{i\in G}x_ix_i^\top\right)(w_{t}-w^*)+\frac{\eta}{n}\sum_{i\in G}x_iz_i+
	\frac{\eta}{n}\sum_{i\in G}(g_i^{(t)}-\tilde{g}_i^{(t)}) - \eta\phi_t\nu_t
	-\frac{\eta}{n}\sum_{i\in S_{\rm bad}}\tilde{g}_i^{(t)} 
	\label{eq:onestep}
\end{align}
Let  $E_t:=\{i\in G: \theta_t\leq |x_i^\top w_t-y_i|\}$ be the set of clipped clean data points such that $\sum_{i\in G}(g_i^{(t)}-\tilde{g}_i^{(t)})=\sum_{i\in E_t }(g_i^{(t)}-\tilde{g}_i^{(t)})$. We define $\hat{v} :=(1/n)\sum_{i\in G}x_iz_i $,   $u_t^{(1)}:=(1/n)\sum_{i\in E_t} x_ix_i^\top (w_t-w^*)$, $u_t^{(2)} := (1/n)\sum_{i\in E_t} -x_iz_i$, and $u_t^{(3)} := (1/n) \sum_{i\in S_{\rm bad}\cup E_t}\tilde{g}_i^{(t)}$. 

We can further write the update rule as:
\begin{align}
    w_{t+1}-w^*
    =&\hat{B}(w_{t}-w^*)+\eta \hat{v}+\eta u_{t-1}^{(1)}+\eta u_{t-1}^{(2)}-\eta \phi_t\nu_t-\eta u_{t-1}^{(3)}\;. \label{eq:update}
\end{align}
We bound each term one-by-one. 
Since $G\subset S_{\rm good}$ and $|G|=(1-\alpha_{\rm corrupt})n$, using the resilience property in \Eqref{def:res1}, we know
\begin{align}
    \|\Sigma^{-1/2}\hat{v}\|&=(1-\alpha_{\rm corrupt})\max_{\|v\|=1}\Sigma^{-1/2}\ip{v}{\frac{1}{(1-\alpha_{\rm corrupt})n}\sum_{i\in G}x_iz_i} \nonumber \\
    &\leq (1-\alpha_{\rm corrupt}) C_2K^2\alpha\log^{2a}(1/\alpha)\sigma\\
    &\leq  C_2K^2\alpha\log^{2a}(1/\alpha)\sigma\;.\label{eq:hat_v}
\end{align}

Let $\tilde{\alpha}=|E_t|/n$. By assumption, we know $\tilde{\alpha}\leq \alpha$ (which holds for the given dataset due to Lemma~\ref{lemma:clipping_fraction}), and  
\begin{align*}
    \|\Sigma^{-1/2} u_t^{(1)}\| &=  \|\Sigma^{-1/2} \frac{1}{n}\sum_{i\in E_t} x_ix_i^\top (w_t-w^*)\|   \;.
\end{align*} 
From Corollary~\ref{coro:res}, we know 
\begin{align*}
    &\left|\|\Sigma^{-1/2} \frac{1}{|E_t|}\sum_{i\in E_t} x_ix_i^\top (w_t-w^*)\|- \|w_t-w^*\|_{\Sigma}\right|\\
    = &\left|\max_{u:\|u\|=1} \frac{1}{|E_t|}\sum_{i\in E_t} u^\top \Sigma^{-1/2} x_ix_i^\top (w_t-w^*)\|- \max_{v:\|v\|=1}v^\top\Sigma^{1/2}(w_t-w^*)\right|\\
    \leq &\max_{u:\|u\|=1}\left| \frac{1}{|E_t|}\sum_{i\in E_t} u^\top  \Sigma^{-1/2}x_ix_i^\top\Sigma^{-1/2} \Sigma^{1/2} (w_t-w^*)\|- u^\top\Sigma^{1/2}(w_t-w^*)\right|\\ 
    \leq &\max_{u:\|u\|=1}\left| \frac{1}{|E_t|}\sum_{i\in E_t} u^\top  \left(\Sigma^{-1/2}x_ix_i^\top\Sigma^{-1/2}-\mathbf{I}_d\right) \Sigma^{1/2} (w_t-w^*)\|\right|\\
    =&\left\| \frac{1}{|E_t|}\sum_{i\in E_t}   \left(\Sigma^{-1/2}x_ix_i^\top\Sigma^{-1/2}-\mathbf{I}_d\right) \Sigma^{1/2} (w_t-w^*)\right\|\\
    \leq &\left\| \frac{1}{|E_t|}\sum_{i\in E_t}   \left(\Sigma^{-1/2}x_ix_i^\top\Sigma^{-1/2}-\mathbf{I}_d\right)\right\|\cdot \left\| \Sigma^{1/2} (w_t-w^*)\right\|\\
    \leq& \frac{2-\tilde{\alpha}}{\tilde{\alpha}}C_2K^2\alpha\log^{2a}(1/\alpha)\left\| w_t-w^*\right\|_\Sigma\;.
\end{align*}
This implies that 
\begin{align}
   \|\Sigma^{-1/2} u_t^{(1)}\|&\;\leq \; \|\Sigma^{-1/2} \frac{1}{n}\sum_{i\in E} x_ix_i^\top (w_t-w^*)\| \nonumber\\
   &\leq \left(\tilde{\alpha}+2C_2K^2\alpha\log^{2a}(1/\alpha)\right)\left\| w_t-w^*\right\|_\Sigma \nonumber\\
   &\leq 3C_2K^2\alpha\log^{2a}(1/\alpha)\left\| w_t-w^*\right\|_\Sigma\;,
   \label{eq:ut1}
\end{align}
where the last inequality follows from  the fact that $\tilde{\alpha}\leq \alpha$ and our assumption that $C_2K^2\log^{2a}(1/\bar{\alpha})\geq 1$ from Assumption~\ref{asmp:corrupt}. 
Similarly, we use resilience property in \Eqref{def:res1} instead of \Eqref{def:res2}, we can show that 
\begin{align}
    \|\Sigma^{-1/2}u_t^{(2)}\| \leq 3C_2K^2\alpha\log^{2a}(1/\alpha)\sigma\;. \label{eq:ut2}
\end{align}

Next, we consider $u_t^{(3)}$. Since $|S_{\rm bad}|\leq \alpha_{\rm corrupt}n$ and $|E_t|\leq \alpha n$,  using \Eqref{def:res4} and Corollary~\ref{coro:res}, we have
\begin{align}
    \|\Sigma^{-1/2}u_t^{(3)}\|&=\max_{v:\|v\|=1}\frac{1}{n}\sum_{i\in S_{\rm bad} \cup E_t}v^\top\Sigma^{-1/2}x_i{\rm clip}_{\theta_t}(x_i^\top w_t-y_i) \nonumber \\
    &\leq 2C_2K\alpha\log^{a}(1/\alpha)\theta_t \nonumber \\
    &\leq 6C_2^{1.5}K^2\alpha\log^{2a}(1/\alpha)(\|w_{t}-w^*\|_\Sigma +\sigma)\;. \label{eq:ut3}
\end{align}

Now we use \Eqref{eq:hat_v}, \Eqref{eq:ut1}, \Eqref{eq:ut2} and \Eqref{eq:ut3} to bound the final error from update rule in \Eqref{eq:update}.

\medskip
{\bf Step 3:  Analysis of the $t$-steps recurrence relation.} We have controlled the deterministic noise in the last step. In this step, we will upper bound the noise introduced by the Gaussian noise for the purpose of privacy, and show the expected distance to optimum decrease every step.

We want to emphasize that most of our technical contribution is in the convergence analysis (Step 3 and Step 4). 
More precisely, naive linear regression analysis can only show a suboptimal error rate of $\|\hat{w} -w^\star\|_\Sigma = \tilde{O}(\kappa \alpha \sigma) $ with sample size $n=\tilde{O}(d/\alpha^2 + \kappa^{1/ 2}d/(\varepsilon \alpha))$. Define $u_{t} = (\hat{v}+ u_{t}^{(1)}+ u_{t}^{(2)}-u_{t}^{(3)} )$. This follows from \Eqref{eq:update}:
\begin{align}
    w_{t+1}-w^*
    =&\hat{B}(w_{t}-w^*)+\eta u_{t} - \eta \phi_t\nu_t \label{eq:update_second}\\
    =&(\mathbf{I}_d-\eta\hat{\Sigma})(w_{t}-w^*)+\eta u_{t} - \eta \phi_t\nu_t \;.
\end{align}
  
From \Eqref{eq:ut1}, \Eqref{eq:ut2} and \Eqref{eq:ut3}, it follows that 
\begin{align*}
\|w_{t+1}-w^*\|_\Sigma \leq (1-\frac{1}{\kappa})\|w_t-w^*\|_\Sigma + \alpha (\sigma+\|w_t-w^*\|_\Sigma)
\end{align*}

where we omitted constants for simplicity, which after $T=\tilde O(\kappa)$ iterations achieves a \emph{sub-optimal} error rate $\|w_{T}-w^*\|_\Sigma = \tilde{O}(\kappa \alpha \sigma)$.

One attempt to get around it is to take the Euclidean norm instead, which gives, after some calculations, 
\begin{align*}
    \E[ \|w_{t+1}-w^*\|^2]  \leq \E[ \|w_t-w^*\|^2] - \eta\Big( \|w_t-w^*\|_\Sigma^2 - \alpha^2\sigma^2 \Big) \;.
\end{align*}

This implies that ${\mathbb E}[\|w_{t+1}-w^*\|^2]$ strictly decreases as long as $\|w_t-w^*\|_\Sigma^2 > C \alpha^2\sigma^2$, which is the desired statistical error level we are targeting. With this analysis, we can show that in  $T=\tilde O(\kappa)$ iterations, there exists at least one model $w_t$ that achieves ${\mathbb E}[\|w_t-w^*\|_\Sigma^2] =\tilde O( \alpha^2\sigma^2)$ among all the intermediate models we have seen. 

However, the problem is that under differential privacy, there is no way we could select this good model $w_t$ among $T$ models that we have, as privacy-preserving techniques for model selection are not accurate enough to achieve the desired level of accuracy. Hence, we came up with the following novel analysis that does not suffer from such issues. 


 We can rewrite \Eqref{eq:update} or \Eqref{eq:update_second} as
\begin{align}
    w_{t+1}-w^*
    =&\hat{B}(w_{t}-w^*)+\eta u_{t} - \eta \phi_t\nu_t \;\\
    =& \hat{B}^{t+1}(w_{0}-w^*)+\eta \sum_{i=0}^{t}\hat{B}^i u_{t-i} - \eta \sum_{i=0}^{t}\phi_{t-i} \hat{B}^i \nu_{t-i}\;.
\end{align} 
Taking expectations of $\hat{\Sigma}$-norm square with respect to $\nu_1, \cdots, \nu_t$, we have
\begin{align}
    &\E_{\nu_1,\ldots,\nu_t\sim \cN(0, \mathbf{I}_d)}\|
w_{t+1}-w^*\|_{\hat{\Sigma}}^2\\
    \leq&\;\; 2\|\hat{B}^{t+1}(w_{0}-w^*)\|_{\hat\Sigma}^2+ 2\E[\|\eta\sum_{i=0}^{t}\hat{B}^i u_{t-i}\|_{\hat{\Sigma}}^2] + \eta^2\sum_{i=0}^{t}\Tr(\hat{B}^{2i}\hat\Sigma)\E[\phi_{t-i}^2]
    \\
\le&\;\; 2\|\hat{B}^{t+1}(w_{0}-w^*)\|_{\hat\Sigma}^2+ 2\eta^2 \E[\sum_{i=0}^{t}\sum_{j=0}^{t}\|\hat{B}^i u_{t-i}\|_{\hat{\Sigma}}\|\hat{B}^j u_{t-j}\|_{\hat{\Sigma}}]\\
&+{\eta^2\sum_{i=0}^{t}\Tr(\hat{B}^{2i}\hat\Sigma)\E[\phi_{t-i}^2]}\;,
\end{align}
where at the second step we used the fact that $\nu_1, \nu_2,\cdots, \nu_t$ are independent isotropic Gaussian. 

Note that
\begin{align*}
\eta\|\hat{B}^i u_{t-i}\|_{\hat{\Sigma}} \;\;=&\;\; \eta\|\hat\Sigma^{1/2}\hat{B}^i\hat\Sigma^{1/2} \hat\Sigma^{-1/2} u_{t-i}\|\\
\le&\;\;\eta\|\hat\Sigma^{1/2}\hat{B}^i\hat\Sigma^{1/2}\|_2\cdot\| \hat\Sigma^{-1/2} u_{t-i}\|\\
\le&\;\; \eta\|\hat\Sigma^{1/2}\hat{B}^i\hat\Sigma^{1/2}\|_2
\,\hat\rho(\alpha)\,(\|w_{t-i}-w^*\|_{\hat\Sigma} +\sigma)\\
\le&\;\; \frac{1}{i+1}\hat\rho(\alpha)\,(\|w_{t-i}-w^*\|_{\hat\Sigma} +\sigma)\;,
\end{align*}
where $\hat\rho(\alpha)=1.1(6C_2+6C_2^{1.5})K^2\alpha\log^{2a}(1/\alpha)$, and the second inequality follows from \Eqref{eq:ut1}, \Eqref{eq:ut2}, \Eqref{eq:ut3} and the deterministic condition in \Eqref{eq:cov_asmp}. Note that the last inequality is true because $\eta\leq 1/(1.1\lambda_{\max})$ and  $\|\hat\Sigma^{1/2}\hat{B}^i\hat\Sigma^{1/2}\|_2\leq \|\mathbf{I}_d-\eta\hat{\Sigma}\|_2^i\|\hat\Sigma\|_2\leq \lambda_{\max}/(i+1)$ .

This implies
\begin{align}
&\E[\eta^2\sum_{i=0}^{t}\sum_{j=0}^{t}\|\hat{B}^i u_{t-i}\|_{\hat{\Sigma}}\|\hat{B}^j u_{t-j}\|_{\hat{\Sigma}}] \\
\le&\;\; 4\,\E[\sum_{i=0}^{t}\sum_{j=0}^{t}\frac{\hat\rho(\alpha)^2}{(i+1)(j+1)}(
\E[\|w_{t-i}-w^*\|_{\hat\Sigma}^2] + \E[\|w_{t-j}-w^*\|_{\hat\Sigma}^2] + \sigma^2)\\
\le&\;\; 8(\sum_{i=0}^t\frac{1}{i+1})^2\hat\rho(\alpha)^2(\max_i\E[\|w_{t-i}-w^*\|_{\hat\Sigma}^2] +  \sigma^2)\\
\le&\;\; 8(\log t)^2\hat\rho(\alpha)^2(\max_i\E[\|w_{t-i}-w^*\|_{\hat\Sigma}^2] +  \sigma^2)\;,
\end{align}
Then, 
\begin{align*}
\|\hat{B}^{t+1}(w_{0}-w^*)\|_{\hat\Sigma}^2 = \|\hat{\Sigma}^{1/2}\hat{B}^{t+1}\hat\Sigma^{-1/2}\hat\Sigma^{1/2}(w_{0}-w^*)\|^2\\
\le (1-\frac{1}{\kappa})^{2(t+1)}\|w_0-w^*\|_{\hat\Sigma}^2\le e^{-2(t+1)/\kappa} \|w_0-w^*\|_{\hat\Sigma}^2\;,
\end{align*}
and for $n\gtrsim (1/\varepsilon)\sqrt{\kappa d \log(1/\delta)/\alpha}$, 
\begin{align}
	&\eta^2\sum_{i=0}^{t}\Tr(\hat{B}^{2i}\hat\Sigma)\E[\phi_{t-i}^2] \\
\leq &\eta^2 \sum_{i=0}^{t}\|\mathbf{I}_d-\eta\hat\Sigma\|_2^{2i}\|\hat{\Sigma}\|_2\cdot\frac{2\log(1.25/\delta_0)K^2\Tr(\Sigma)\log^{2a}(n/\zeta_0)C_2K^2\log^{2a}(1/(2\alpha))(\E[\|w_{t-i}-w^*\|_{\Sigma}^2]+\sigma^2)}{\varepsilon_0^2n^2}\\
\leq &4\sum_{i=0}^t(\frac{1}{i+1})^2\hat\rho(\alpha)^2(\E[\|w_{t-i}-w^*\|_{\hat\Sigma}^2] +  \sigma^2)\;. \label{eq:noise_ineq_main}
\end{align}
We have
$$
\E_{\nu_1,\ldots,\nu_t\sim \cN(0, \mathbf{I}_d)}[\|
w_{t+1}-w^*\|_{\hat{\Sigma}}^2] \le 2e^{-2(t+1)/\kappa} \|w_0-w^*\|_{\hat\Sigma}^2+ 20(\log t)^2\hat\rho(\alpha)^2(\max_{i\in [t]}\E[\|w_{t-i}-w^*\|_{\hat\Sigma}^2] + \sigma^2)\;.
$$

Note that this also implies that 
\begin{align}
\E[\|
(w_{t'+t}-w^*)\|_{\hat{\Sigma}}^2|w_{t'}] \le 2e^{-2t/\kappa} \|w_{t'}-w^*\|_{\hat\Sigma}^2+ 20
\hat\rho(\alpha)^2\sum_{i=0}^{t-1} (\frac{1}{i+1})^2 (\E[\|w_{t'+t-i}-w^*\|_{\hat\Sigma}^2|w_{t'}] + \sigma^2)\;,
\end{align}
which  implies
\begin{align}
 \E[\|
(w_{t'+t}-w^*)\|_{\hat{\Sigma}}^2] &\le2 e^{-2t/\kappa} \E[\|w_{t'}-w^*\|_{\hat\Sigma}^2]+ 
20
\hat\rho(\alpha)^2\sum_{i=0}^{t-1} (\frac{1}{i+1})^2 (\E[\|w_{t'+t-i}-w^*\|_{\hat\Sigma}^2] + \sigma^2)\\
&\le2 e^{-2t/\kappa} \E[\|w_{t'}-w^*\|_{\hat\Sigma}^2]+ 
20
(\log t)^2\hat\rho(\alpha)^2 (\max_{i\in [t]}\E[\|w_{t'+t-i}-w^*\|_{\hat\Sigma}^2] + \sigma^2)
\end{align} 
{\bf Step 4: End-to-end analysis of the convergence.}
In the last step, we shown that the amount of estimation error decrease depends on the estimation error of the previous $t$ steps. In order for the estimation error to decrease by a constant factor, we will take $t=\kappa$. Roughly speaking, we will prove that for every $\kappa$ steps, the estimation error will decrease by a constant factor, if it is much larger than $O((\log\kappa)^2\hat\rho(\alpha)^2\sigma^2)$. This implies we will reach $O((\log\kappa)^2\hat\rho(\alpha)^2\sigma^2)$ error with in $\tilde{O}(\kappa)$ steps.

For any integer $s\ge 0$, as long as 
$\max_{i\in [(s-1)\kappa+1, s\kappa]} \E[\|w_{i}-w^*\|_{\hat\Sigma}^2]\ge 2(\log\kappa)^2\hat\rho(\alpha)^2\sigma^2$, 
\begin{eqnarray}
\max_{i\in [s\kappa+1, (s+1)\kappa]} \E[\|w_{i}-w^*\|_{\hat\Sigma}^2] \le (\frac{1}{e^2}+(\log \kappa)^2\hat\rho(\alpha)^2)\max_{i\in [(s-1)\kappa+1, s\kappa]} \E[\|w_{i}-w^*\|_{\hat\Sigma}^2] + (\log 2\kappa)^2\hat\rho(\alpha)^2\sigma^2 \;.
\end{eqnarray}

Assuming $\hat\rho(\alpha)^2(\log \kappa )^2\le 1/2-1/e^2$, the maximum expected error in a length $\kappa$ sequence decrease by a factor of $1/2$ every time. 

Now we bound the maximum expected error in the first length $\kappa$ sequence:
$
\max_{i\in [0, \kappa-1]} \E[\|w_{i}-w^*\|_{\hat\Sigma}^2].
$
Since 
$$
\E[\|w_{i}-w^*\|_{\hat\Sigma}^2] \le e^{-2i/\kappa}\|w_0-w^*\|_{\hat\Sigma}^2 + (\log i )^2 \hat\rho(\alpha)^2
\max_{j\in [0, i-1]} \E[\|w_{j}-w^*\|_{\hat\Sigma}^2] + (\log i )^2 \hat\rho(\alpha)^2\sigma^2\;.
$$
As a function of $i$, 
$\max_{j\in [0, i-1]} \E[\|w_{j}-w^*\|_{\hat\Sigma}^2]$ only increase when it is smaller than 
$$
\frac{1}{1-(\log i)^2 \hat\rho(\alpha)^2}(\|w_0-w^*\|_{\hat\Sigma}^2 + (\log i)^2 \hat\rho(\alpha)^2\sigma^2)\;.
$$
Thus we conclude 
$$
\max_{i\in [0, \kappa-1]} \E[\|w_{i}-w^*\|_{\hat\Sigma}^2] \le \frac{1}{1-(\log \kappa)^2 \hat\rho(\alpha^2)}(\|w_0-w^*\|_{\hat\Sigma}^2 + (\log \kappa)^2 \hat\rho(\alpha^2) \sigma^2)
$$
$s=\log(\|w^*\|/(\hat\rho(\alpha)\sigma))$ will give us 
$$
\E[\|w_{s\kappa+1}-w^*\|_{\hat\Sigma}^2] \le (\log \kappa)^2 \hat\rho(\alpha)^2 \sigma^2\;.
$$

\end{proof}

\section{Lower bounds}

\subsection{Proof of Proposition~\ref{propo:lb} for label corruption lower bounds}\label{sec:label-lower-bounds}

We first prove the following lemma.
\begin{lemma}
\label{lemma:lb}
Consider  an $\alpha$ label-corrupted dataset $S = \{(x_i, y_i)\}_{i=1}^n$ with $\alpha < 1/2$, that is  generated from either 
$x_i\sim \cN(0, 1), y_i\sim \cN(0,1)$
or 
$x_i\sim \cN(0,1), z_i\sim \cN(0, 1-\alpha^2), y_i =  \alpha x_i+z_i.$ It is impossible to distinguish  the two hypotheses with probability larger than $1/2$. 
\end{lemma}
In the first case,
$$
(x_i, y_i)\sim \cP_1 = \cN(0, \begin{bmatrix}1&0\\0&1\end{bmatrix}).
$$
In the second case,
$$
(x_i, y_i)\sim \cP_2 = \cN(0, \begin{bmatrix}1& \alpha \\  \alpha & 1\end{bmatrix}).
$$
By simple calculation, it holds that $D_{KL}(\cP_1||\cP_2) = -\frac{1}{2}\log(1-\alpha^2)\le \alpha^2/2$ for all $\alpha<1/2$. Then, Pinsker's inequality implies that $D_{TV}(\cP_1||\cP_2) \le \alpha/2$. Since the covariate $x_i$ follows from the same distribution in the two cases, and the total variation distance between the two cases is less than $\alpha/2$. This means there is an label corruption adversary that change $\alpha/2$ fraction of $y_i$'s in $\cP_1$ to make it identical to $\cP_2$. Therefore, no algorithm can distinguish the two cases with probability better than $1/2$ under $\alpha$ fraction of label corruption.

Since $\Sigma=1$, $\sigma^2\in [3/4,1]$, the first case above has $w^*=0$, and the second case has $w^*=\alpha$, this implies that no algorithm is able to achieve $\E[\|\hat{w}-w^*\|_{\Sigma}] < \sigma \alpha$ for all instances with $\|w^*\|\le 1$ under $\alpha$ fraction of label corruption.

\section{Technical Lemmas}
\label{app:technical}

\begin{lemma}[Hanson-Wright inequality for subWeibull distributions \cite{sambale2020some}]
\label{lemma:hanson}
Let $S=\{x_i\in \reals^d\}_{i=1}^n$ be a dataset consist of i.i.d. samples from $(K,a)$-subWeibull distributions, then
\begin{align}
	\prob\left(\left|\frac{1}{n}\sum_{i=1}^n\|x_i\|^2-\Tr(\Sigma)\right|\geq t\right)\leq 2\exp\left(-\min\left\{\frac{n t^2}{K^4(\Tr(\Sigma))^2}, \left(\frac{nt}{K^2\Tr(\Sigma)}\right)^{\frac{1}{2a}}\right\}\right)\;.
\end{align}	
\end{lemma}

\begin{lemma}
\label{lemma:lap_noise}
Let $Y\sim {\rm Lap}(b)$. Then for all $h>0$, we have $\prob(|Y|\geq hb) = e^{-h}$.
\end{lemma}

\begin{lemma}
	\label{lemma:norm_a_tail}
	If $x\in \reals^d$ is $(K,a)$-subWeibull for some $a\in [1/2, \infty)$. Then 
	\begin{itemize}
		\item for any fixed $v\in \reals^d$, with probability $1-\zeta$,
	\begin{align}
		\ip{x}{v}^2\leq K^2v^\top \Sigma v\log^{2a}(1/\zeta)\;.
	\end{align}
	\item with probability $1-\zeta$,
	\begin{align}
		\|x\|^2\leq K^2 \Tr(\Sigma)\log^{2a}(1/\zeta)\;.
	\end{align}
	\end{itemize}
\end{lemma}
We provide a proof in Appendix~\ref{app:proof_norm_a_tail}.

\begin{lemma}\label{lemma:cov_concentration_tail}
Dataset $S=\{x_i\in \reals^d\}_{i=1}^n$ consists i.i.d. samples from a zero mean  distribution $\cD$. Suppose $\cD$ is $(K,a)$-subWeibull. Define $\Sigma=\E_{x\sim \cD}[xx^\top]$. Then there exists a constant $c_1>0$ such that with probability $1-\zeta$,
\begin{align}
	\left\|\frac{1}{n}\sum_{i=1}^nx_ix_i^\top-\Sigma\right\|\leq c_1\left(\frac{K^2d\log(d/\zeta)\log^{2a}(n/\zeta)}{n}+\sqrt{\frac{K^2d\log(d/\delta)\log^{2a}(n/\zeta)}{n}}\right)\|\Sigma\|_2\;.
\end{align}
\end{lemma}

\begin{lemma}[Lemma~F.1 from \cite{liu2022dp}]
\label{lemma:gauss_norm}
	Let $x\in \reals^d\sim \cN(0, \Sigma)$. Then there exists universal constant $C_6$ such that with probability $1-\zeta$,
	\begin{align}
		\|x\|^2\leq C \Tr(\Sigma)\log(1/\zeta)\;.
	\end{align}
\end{lemma}

\begin{definition}[Corrupt good set]
    \label{def:corruptgood}
	We say a dataset $S$ is $(\alpha_{\rm corrupt}, \alpha,\rho_1,\rho_2, \rho_3, \rho_4)$-corrupt good with respect to $(w^*,\Sigma,\sigma)$ if it is $\alpha_{\rm corrupt}$-corruption of an  $(\alpha,\rho_1,\rho_2, \rho_3, \rho_4)$-resilient dataset $S_{\rm good}$.
\end{definition}

\begin{lemma} 
\label{lemma:subweibull_res_conditions}
Under Assumptions~\ref{asmp:distribution} and \ref{asmp:corrupt}, there exists positive constants $c_1$ and $C_2$ such that if $n\geq c_1((d+\log(1/\zeta)) / \alpha^2$, then with  probability $1-\zeta$, 
     $S_{\rm good}$ is,  with respect to $(w^*,\Sigma,\sigma)$, $(\alpha, C_2K^2\alpha\log^{2a}(1/\alpha),C_2K^2\alpha\log^{2a}(1/\alpha),C_2K^2\alpha\log^{2a}(1/\alpha), C_2K\alpha\log^{a}(1/\alpha))$-resilient.
\end{lemma}
We provide a proof in Appendix~\ref{app:proof_subweibull_res_conditions}.

\begin{coro}[Lemma~10 from \cite{steinhardt2017resilience} and Lemma~25 from \cite{liu2022differential}] 
\label{coro:res} 
    For a  $(\alpha,\rho_1,\rho_2, \rho_3,\rho_4)$-resilient set $S$ with respect to $(w^*,\Sigma,\gamma)$ and any $0\leq \tilde\alpha \leq  \alpha$,  the following holds  for any subset $T\subset S$ of size at least $\tilde\alpha n$ and for any unit vector  $v\in \reals^d$: 
    \begin{eqnarray}
     \Big|\frac{1}{|T|}\sum_{(x_i,y_i)\in T} \langle v,x_i\rangle (y_i-x_i^\top w^*)\Big| &\leq & \frac{2-\tilde\alpha}{\tilde\alpha}\rho_1 \,\sqrt{v^\top\Sigma v}\,\sigma \;, 
    \label{eq:res_small1}\\ 
    \left|\frac{1}{|T|}\sum_{x_i\in T}\, \langle v, x_i\rangle^2-v^\top\Sigma v \,\right| &\leq& \frac{2-\tilde\alpha}{\tilde\alpha}\,\rho_2 v^\top\Sigma v\;, \label{eq:res_small2}\\
    \Big| \frac{1}{|T|}\sum_{(x_i, y_i)\in T} (y_i-x_i^\top w^*)^2-\sigma^2  \Big|  & \leq &  \frac{2-\tilde{\alpha}}{\tilde{\alpha}}\rho_3 \, \sigma^2    \;\label{eq:res_small3},\; \text{ and }\\
    \left|\frac{1}{|T|}\sum_{x_i\in T}\, \langle v, x_i\rangle \,\right| &\leq& \frac{2-\tilde\alpha}{\tilde\alpha}\,\rho_4  \sqrt{v^\top\Sigma v}\label{eq:res_small4} \;.
    \end{eqnarray}
\end{coro}

\subsection{Proof of technical lemmas}
\subsubsection{Proof of Lemma~\ref{lemma:norm_a_tail}}
\label{app:proof_norm_a_tail} 

Using Markov inequality, 
	\begin{align}
		\prob\left(\ip{v}{x}^2\geq t^2\right) 
		&= \prob\left(e^{\ip{v}{x}^{1/a}}\geq e^{t^{1/a}}\right)\\
		&\leq e^{-t^{1/a}}\E[e^{\ip{v}{x}^{1/a}}]\\
		&\leq e^{-t^{1/a}}e^{K (\E[\ip{v}{x}^2])^{1/(2a)}}\\
		&=2\exp\Big(-\Big(\frac{t^2}{K^2\E[\ip{v}{x}^2]}\Big)^{1/(2a)}\Big)\;. \label{eq:tail_weibull}
	\end{align}
	
	This implies for any fixed $v$, with probability $1-\zeta$,
	\begin{align}
		\ip{x}{v}^2\leq K^2v^\top \E[xx^\top]v\log^{2a}(1/\zeta)\;.
	\end{align}
	
	For $j$-th coordinate, let $v=e_j$ where $j\in [d]$.  Definition~\ref{def:a_tail} implies
	\begin{align}
		\E\left[\exp\left(\left(\frac{x_j^2}{K^2\Tr(\Sigma)}\right)^{1/(2a)}\right)\right]\leq \E\left[\exp\left(\left(\frac{x_j^2}{K^2\Sigma_{jj}}\right)^{1/(2a)}\right)\right]\leq 2\;.
	\end{align}	
	
	Note that $f(x)=x^{\alpha}$ is concave function for $\alpha\leq 1$ and $x>0$. Then $(a_1+\cdots a_k)^\alpha\leq a_1^\alpha+ \cdots a_k^\alpha$ holds for any positive numbers $a_1, \cdots, a_k >0$. By our assumption that $1/(2a)\leq 1$. , we have
	\begin{align}
		\E[\exp\left({\left(\frac{\|x\|^2}{K^2\Tr(\Sigma)}\right)}^{1/(2a)}\right)]&=\E[\exp\left({\left(\frac{x_1^2+x_2^2+\cdots+x_d^2}{K^2\Tr(\Sigma)}\right)}^{1/(2a)}\right)]\\
		&\leq \E[\prod_{j=1}^d\exp\left(\left(\frac{x_j^2}{K^2\Tr(\Sigma)}\right)^{1/(2a)}\right)]\\
		&\leq \left(\frac{\sum_{j=1}^d\E[\exp\left(\left(\frac{x_j^2}{K^2\Tr(\Sigma)}\right)^{1/(2a)}\right)]}{d}\right)^d\\
		&\leq 2\;.
	\end{align}
	
	By Markov inequality,
	\begin{align}
		\prob\left(\|x\|\geq t\right) 
		&= \prob\left(e^{\|x\|^{1/a}}\geq e^{t^{1/a}}\right)\\
		&\leq e^{-t^{1/a}}\E[e^{\|x\|^{1/a }}]\\
		&\leq \exp\left(-\left(\frac{t^2}{K^2\Tr(\Sigma)}\right)^{1/(2a)}\right)\;.
	\end{align}
	This implies with probability $1-\zeta$,
	\begin{align}
		\|x\|^2\leq K^2\Tr(\Sigma)\log^{2a}(1/\zeta)\;.
	\end{align}
	
	


\section{Experiments}
\label{appendix:experiments}
\subsection{DP Linear Regression}
Experimental results for $\epsilon=0.1$ can be found in Figure~\ref{fig:dp_regression_more}. The observations are similar to the $\epsilon=1$ case. In particular, $\ssp$ has poor performance when $\sigma$ is small. In other settings, $\ssp$ has better performance than $\dprobgd$. 
\begin{figure}
    \centering
    \includegraphics[scale=0.25]{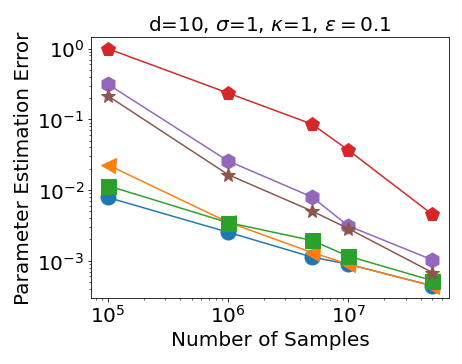}
    \includegraphics[scale=0.25]{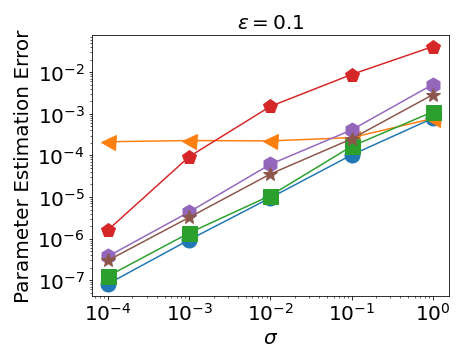}
    \includegraphics[scale=0.25]{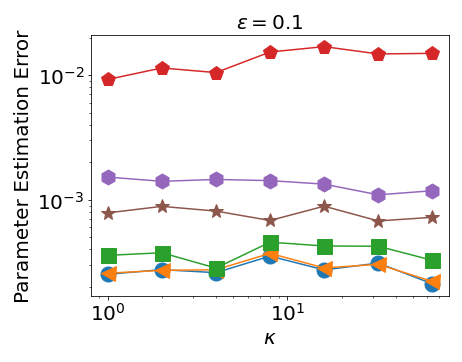}
    \includegraphics[scale=0.3]{plots/legend_trimmed_new.png}
    \vspace{-0.1in}\caption{Performance of various techniques on DP linear regression. $d=10$ in all the experiments. $n=10^7, \kappa=1$ in the $2^{nd}$ experiment. $n=10^7, \sigma=1$ in the $3^{rd}$ experiment.} 
    \label{fig:dp_regression_more}
    \vspace{-0.1in}
\end{figure}

\subsection{DP Robust Linear Regression}

We now illustrate the robustness of our algorithm. We consider the same experimental setup as in Section~\ref{sec:exp} and randomly corrupt $\alpha$ fraction of the response variables by setting them to $1000$. 
 Figure~\ref{fig:dp_robust_regression} presents the results from this experiment. It can be seen that none of the baselines are robust to adversarial corruptions. They can be made arbitrarily bad by increasing the magnitude of corruptions. In contrast, $\dprobgd$ is able to handle the corruptions well. 
\begin{figure}
    \centering
\includegraphics[scale=0.25]{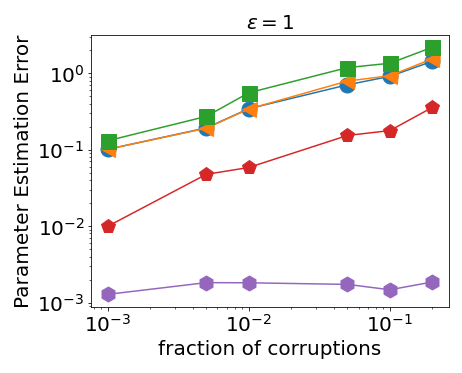}
    \includegraphics[scale=0.25]{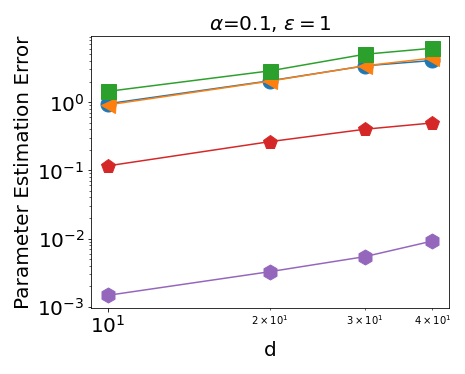}
    \includegraphics[scale=0.3]{plots/legend_trimmed_new.png}
    \caption{Non-robustness of existing techniques to adversarial corruptions. $n=10^7$, $\sigma=1$ in both experiments.}
    \label{fig:dp_robust_regression}
\end{figure}

\subsection{Stronger adversary for DP Robust Linear Regression}
In this section, we consider a stronger adversary for $\dprobgd$ than the one considered in Section~\ref{sec:exp}. Recall, for the adversary model considered in Section~\ref{sec:exp}, $\dprobgd$ was able to consistently estimate the parameter $w^*$ (\emph{i.e.,} the parameter recovery error goes down to $0$ as $n\to\infty$). This is because the algorithm was able to easily identify the corruptions and ignore the corresponding points while performing gradient descent. We now construct a different instance where the corruptions are hard to identify. Consequently, $\dprobgd$ can no longer be consistent against the adversary. This hard instance is inspired by the lower bound in \citet{bakshi2021robust} (see Theorem 6.1 of \citet{bakshi2021robust}). This is a 2 dimensional problem where the first covariate is sampled uniformly from $[-1,1].$ The second covariate, which is uncorrelated from the first, is sampled from a distribution with the following pdf
\[
p(x^{(2)}) = \begin{cases}
\frac{\alpha}{2} \quad &\text{if } x^{(2)} \in \{-1,1\}\\
\frac{1-\alpha}{2\alpha\sigma} \quad & \text{if } x^{(2)} \in [-\sigma, \sigma]\\
0 \quad & \text{otherwise}
\end{cases}.
\]
We set $\sigma=0.1$ in our experiments. The noise $z_i$ is sampled uniformly from $[-\sigma, \sigma].$ We consider two possible parameter vectors $w^* = (1,1)$ and $w^*=(1,-1)$. It can be shown that the total variation (TV) distance between these problem instances (each parameter vector corresponds to one problem instance) is $\Theta(\alpha)$~\citep{bakshi2021robust}. What this implies is that, one can corrupt at most $\alpha$ fraction of the response variables and convert one problem instance into another. Since the distance (in $\Sigma$ norm) between the two parameter vectors is $\Omega(\alpha \sigma),$ any algorithm will suffer an error of $\Omega(\alpha \sigma)$. 

We generate $10^7$ samples from this problem instance and add corruptions that convert one problem instance to the other. Figure~\ref{fig:dp_robust_regression_more} presents the results from this experiment. It can be seen that our algorithm works as expected. In particular, it is not consistent in this setting. Moreover, the parameter recovery error increases with the fraction of corruptions.
\begin{figure}
    \centering
    \includegraphics[scale=0.25]{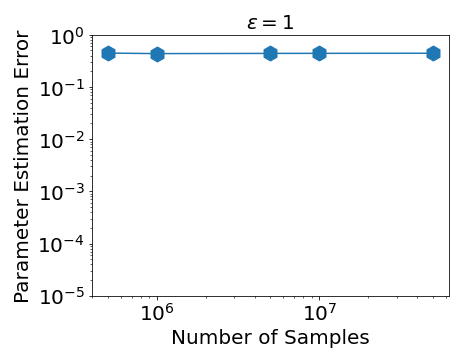}
    \includegraphics[scale=0.25]{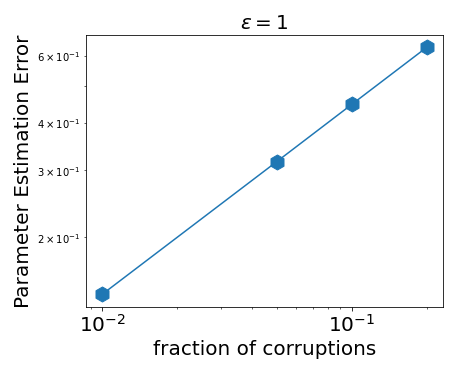}
    \caption{Performance against the stronger adversary}
    \label{fig:dp_robust_regression_more}
\end{figure}

\paragraph{}

\section{Heavy-tailed noise}
\label{app:heavy_tailed}
We study the heavy-tailed regression settings where the label noise $z_i$ is hypercontractive, which is common in robust linear regression literature \citep{klivans2018efficient,liu2022differential}. We define $(\kappa_2,k)$-hypercontractivity as follows. This is a heavy-tailed distribution we have bound only up to the $k$-th moment. 

\begin{definition}
     \label{def:hyper}
     For integer $k\geq 4$, a distribution $P_{\mu,\Sigma}$ is 
     $(\kappa_2,k)$-hypercontractive if for all $v\in\reals^d$, 
     $\E_{x\sim P_X}[ |\langle v , (x-\mu)\rangle |^k] \leq \kappa_2^k (v^\top\Sigma v)^{k/2}$, where $\Sigma$ is the covariance.
 \end{definition} 

We give a formal description of our setting in Assumption~\ref{asmp:distribution_ht}. Note that we consider the input vector $x_i$ to be sub-Weibull and label noise $z_i$ to be hypercontractive. If both $x_i$ and $z_i$ are hypercontractive, the uncorrupted set $S_{\rm good}$ is known to be not resilient \citep{zhu2019generalized,liu2022differential}. However, by \citep[Lemma~G.10]{zhu2019generalized}, we can clip $x_i$ by $O(\sqrt{d}\|\Sigma\|_2)$, and obtain a $(\alpha, O(\kappa\alpha^{1-1/k}), O(\kappa\alpha^{1-2/k}), O(\kappa\alpha^{1-2/k}), O(\kappa\alpha^{1-1/k}))$-resilient set \citep[Lemma~4.19]{liu2022differential}. This would result in sub-optimal error rate $\tilde{O}(\kappa\alpha^{1-2/k})$, which depends on condition number $\kappa$. For convenience, in this section, we further assume that  $x_i$ and $z_i$ are independent. In the dependent case, the only thing we need to change is the $\rho_1$ resilience from $O(\alpha^{1-1/k})$ to $O(\alpha^{1-2/k})$ in Lemma~\ref{lemma:ht_conditions}. This would result in $O(\alpha^{1-3/k})$ error rate if we plug  this new resilience in Theorem~\ref{thm:main_ht}.

\begin{asmp}[$(\Sigma,\sigma^2,w^*,K,a, \kappa_2, k)$-model]
    \label{asmp:distribution_ht}
	A multiset $S_{\rm good}= \{(x_i\in \reals^d, y_i\in \reals)\}_{i=1}^n$ of $n$ i.i.d.~samples  is from a linear model $y_i=\ip{x_i}{w^*}+z_i$, where the input vector $x_i$ is zero mean, $\E[x_i]=0$, with a positive definite covariance  $\Sigma:=\E[x_ix_i^\top]\succ 0$, and the independent label noise $z_i$ is zero mean, $\E[z_i]=0$, with variance $\sigma^2 := \E[z_i^2]$.
	We assume that the marginal distribution of $x_i$ is $(K,a)$-sub-Weibull and that of $z_i$ is  $(\kappa_2,k)$-hypercontractive, as defined above.  
	\end{asmp}
	This is similar to the light-tailed case in Assumption~\ref{def:a_tail}. The main difference is that the noise $z_i$ is heavy-tailed and independent of the input $x_i$. 
	\begin{asmp}[$\alpha_{\rm corrupt}$-corruption]  \label{asmp:corrupt_ht} 
Given a dataset $S_{\rm good}=\{(x_i, y_i)\}_{i=1}^n$, an adversary inspects all the data points, selects $\alpha_{\rm corrupt} n$ data points denoted as $S_r$, and replaces the labels with arbitrary labels while keeping the covariates unchanged. 
We let $S_{\rm bad}$ denote this set of $\alpha_{\rm corrupt} n$ newly labelled examples by the adversary. Let the resulting set be $S:=S_{\rm good}\cup S_{\rm bad}\setminus S_r$.  We further assume that the corruption rate is bounded by $\alpha_{\rm corrupt} \leq \bar \alpha$, where $\bar\alpha$ is a positive constant that depends on $\kappa_2$, $k$, $K$, $\log(\kappa)$, $a$ and $\zeta$.
\end{asmp}

Compared to Assumption~\ref{asmp:corrupt}, this only difference is in the conditions on $\bar\alpha$.  Similar as Lemma~\ref{lemma:subweibull_res_conditions}, we have the following lemma showing that under Assumption~\ref{asmp:distribution_ht}, the uncorrupted dataset can $S_{\rm good}$ is corrupt-good, which means that it can be seen as being corrupted from a resilient set. We provide the proof in App.~\ref{app:proof_ht_conditions}.

 \begin{lemma}
	\label{lemma:ht_conditions}
	A multiset of i.i.d.~labeled samples $S_{\rm good}=\{(x_i, y_i )\}_{i=1}^n$ is generated from a linear model: $y_i=\ip{x_i}{w^*}+z_i$,  where feature vector $x_i$ has zero mean and covariance $\E[x_ix_i^\top]=\Sigma \succ 0$, independent label noise $z_i$ has zero mean and covariance $\E[z_i^2]=\sigma^2>0$.
		Suppose $x_i$ is $(K,a)$-sub-Weibull, $z_i$ is  $(\kappa_2, k)$-hypercontractive, then there exist constants $c_1, C_2>0$ such that, for any  $0<\alpha\leq \tilde \alpha\leq c$ where $c\in (0,1/2)$ is some absolute constant if 
	\begin{align}
	n \geq c_{1}\left(\frac{d}{\zeta^{2(1-1 / k)} \alpha^{2(1-1 / k)}}+\frac{k^2 \alpha^{2-2 / k} d \log d}{\zeta^{2-4 / k} \kappa_2^2}+\frac{\kappa_2^2 d \log d}{\alpha^{2 / k}}+\frac{d+\log(1/\zeta)}{\tilde{\alpha}^2}\right)\;,
	\end{align}		
	then with probability $1-\zeta$, 
	$S_{\rm good}$ is \\
		 $(0.2\alpha, \alpha, C_2k(ka)^aK\kappa_2\alpha^{1-1/k}\zeta^{-1/k},C_2K^2\tilde{\alpha}\log^{2a}(1/\tilde{\alpha}),C_2k^2\kappa_2^2\alpha^{1-2/k}\zeta^{-2/k}, C_2K\tilde\alpha \log^{a}(1/\tilde{\alpha}))$-corrupt good with respect to $(w^*,\Sigma,\sigma)$.
	\end{lemma}
 	
In the rest of this section, we assume we have a $(O(\alpha),\alpha,\rho_1,\rho_2, \rho_3, \rho_4)$-corrupt good set under Assumption~\ref{asmp:distribution_ht} and present following algorithm and our main theorem under this setting in Theorem~\ref{thm:main_ht}. We also provide the proof in App.~\ref{app:proof_main_ht}.

 \begin{algorithm2e}[H]    
   \caption{Robust and Private Linear Regression for heavy-tailed noise}
   \label{alg:main_ht}
   	\DontPrintSemicolon 
	\KwIn{dataset $S=\{(x_i, y_i)\}_{i=1}^{3n}$, $(\varepsilon ,\delta )$,  $T$, learning rate $\eta$, failure probability $\zeta$,  target error rate $\alpha$, distribution parameter $(K,a)$}
	\SetKwProg{Fn}{}{:}{}
	{ 
	Partition dataset $S$ into three equal sized disjoint subsets $S =  S_1\cup S_2\cup S_3  $. \\
	$\delta_0\gets\delta/(2T)$, $\varepsilon_0\gets\varepsilon/(4\sqrt{T\log(1/\delta_0)})$, $\zeta_0\gets \zeta/3$, 	$w_0\gets 0$\\
	$\Gamma\gets {\rm Private Norm Estimator}(S_1, \varepsilon_0, \delta_0, \zeta_0)$, 
    $\Theta\gets K\sqrt{2\Gamma}\log^{a}(n/{ \zeta_0})$\label{line:clip0_ht}\\
	\For{$t=1, 2, \ldots,  T-1$}{ 
	$ \gamma_t \gets {\rm Robust Private Distance Estimator}(S_2, w_t, \varepsilon_0, \delta_0, \alpha, \zeta_0)$ \\
	$\theta_t \gets 2 \sqrt{2\gamma_t}  \cdot \sqrt{\max\{8\rho_2/\alpha, 8\rho_3/\alpha\}+1}\label{line:clip1_ht}$.\\
	Sample $\nu_t \sim \cN\left(0, \mathbf{I}_d\right)$\\
	$w_{t+1}\gets   w_{\color{blue} t}-\eta\left(\frac{1}{n}\sum_{i\in S_3}\left({\rm clip}_{\Theta}(x_{i}){\rm clip}_{\theta_t}\left(w_t^\top x_i-y_i\right)\right)+\frac{\sqrt{2\log(1.25/\delta_0)}\Theta\theta_t}{\varepsilon_0 n}\cdot\nu_t\right)	$ \\
	}
	Return $w_T$
	}
\end{algorithm2e}

\begin{thm}
    \label{thm:main_ht}
    Algorithm~\ref{alg:main_ht} is $(\varepsilon, \delta)$-DP. 
Under $(\Sigma,\sigma^2,w^*,K,a, \kappa_2, k)$-model of Assumption~\ref{asmp:distribution_ht} and $\alpha_{\rm corrupt}$-corruption of Assumption \ref{asmp:corrupt_ht} and for any failure probability $\zeta\in(0,1)$ and target error rate $\alpha \geq 1.2\alpha_{\rm corrupt}$, if the dataset $S$ is  $(0.2\alpha,\alpha,\rho_1,\rho_2, \rho_3, \rho_4)$-corrupt good set $S$ with respect to $(w^*,\Sigma,\sigma)$ and sample size is large enough such that 
	\begin{align}
	    n = 
	    & O  \left(K^2d\log(d/\zeta)\log^{2a}(n/\zeta)
	    +
	    \frac{K^2 d T^{1/2}\log(T/\delta)\log^a(n/(\alpha\zeta))\sqrt{8\max\{\rho_2/\alpha,\rho_3/\alpha\}+1}}{\varepsilon \hat\rho(\alpha) }\right) ,
	    \label{eq:main_n_hyper} 
	\end{align} 
	where  $\hat{\rho}(\alpha)=\max\{\rho_1, 3\rho_2, 2\rho_4\sqrt{8\max\{\rho_2/\alpha,\rho_3/\alpha\}+1}\}$, then the choices of a small enough step size, $\eta\leq 1/(1.1\lambda_{\max}(\Sigma))$, and the number of iterations,  
	    $T= \tilde\Theta\left(\kappa \log\left(\|w^*\|
	    \right)\right)\,$ for a condition number of the covariance $\kappa:=\lambda_{\rm max}(\Sigma)/\lambda_{\rm min}(\Sigma)$, 
ensures that, with probability $1-\zeta$, Algorithm~\ref{alg:main} achieves  
	\begin{align}
		 &\E_{\nu_1, \cdots, \nu_t\sim \cN(0, \mathbf{I}_d)}\big[\,\|  w_{T}-w^*\|_\Sigma^2\,\big] \;\;=\;\; 
		  \tilde{O}\Big(\,  \hat{\rho}^2(\alpha)\sigma^2\,\Big)\;,
		 \label{eq:main_hyper} 
	\end{align}
	where the expectation is taken over the noise added for DP, and $\tilde \Theta(\cdot)$ hides logarithmic terms in $K,\kappa_2, \sigma,d,n,1/\varepsilon,\log(1/\delta),1/\alpha$, and $\kappa$. 
\end{thm}

By Lemma~\ref*{lemma:ht_conditions}, if we set $\tilde{\alpha}=\alpha^{1-1/k}$, $\rho_1=C_2k(ka)^aK\kappa_2\alpha^{1-1/k}\zeta^{-1/k}$, $\rho_2=C_2K^2{\alpha^{1-1/k}}\log^{2a}(1/{\alpha^{1-1/k}})$,$\rho_3= C_2k^2\kappa_2^2\alpha^{1-2/k}\zeta^{-2/k}$, and $\rho_4=C_2K\alpha^{1-1/k} \log^{a}(1/{\alpha^{1-1/k}})$, we have following corollary.

\begin{coro}
\label{coro:ht_robust}
	 Under the same hypotheses of Theorem~\ref{thm:main_ht} and under $\alpha_{\rm corrupt}$-corruption model of Assumption~\ref{asmp:corrupt_ht}, if $1.2\alpha_{\rm corrupt}\leq \alpha$ and $K,a,\kappa_2,k=O(1)$, then 
	 $n=\tilde O(d/(\zeta^{2-2/k}\alpha^{2-2/k}) + \kappa^{1/2} d/ (\varepsilon \alpha^{1-1/k}))$ samples are sufficient for Algorithm~\ref{alg:main_ht} to achieve an error rate of $(1/\sigma^2)\|\hat w - w^*\|_\Sigma^2 = \tilde O(\zeta^{-2/k}\alpha^{2-4/k})$ with probability $1-\zeta$, where $\kappa:=\lambda_{\rm max}(\Sigma)/\lambda_{\rm min}(\Sigma)$, $\tilde O(\cdot)$ hides logarithmic terms in $\sigma, d, n, 1/\varepsilon, \log(1/\delta), \log(1/\zeta)$ and $\kappa$.
\end{coro}

Simiarly, if we set $\tilde{\alpha}=\alpha$, $\rho_1=C_2k(ka)^aK\kappa_2\alpha^{1-1/k}\zeta^{-1/k}$, $\rho_2=C_2K^2{\alpha}\log^{2a}(1/{\alpha})$,$\rho_3= C_2k^2\kappa_2^2\alpha^{1-2/k}\zeta^{-2/k}$, and $\rho_4=C_2K\alpha \log^{a}(1/{\alpha})$, we have following corollary.

\begin{coro}
\label{coro:ht_robust2}

	Under the same hypotheses of Theorem~\ref{thm:main_ht} and under $\alpha_{\rm corrupt}$-corruption model of Assumption~\ref{asmp:corrupt_ht}, if $1.2\alpha_{\rm corrupt}\leq \alpha$ and $K,a,\kappa_2,k=O(1)$, then 
	$n=\tilde O(d/(\zeta^{2-2/k}\alpha^{2-2/k}) + \kappa^{1/2} d/ (\varepsilon \alpha)+(d+\log(1/\zeta)/\alpha^2))$ samples are sufficient for Algorithm~\ref{alg:main_ht} to achieve an error rate of $(1/\sigma^2)\|\hat w - w^*\|_\Sigma^2 = \tilde O(\zeta^{-2/k}\alpha^{2-2/k})$ with probability $1-\zeta$, where $\kappa:=\lambda_{\rm max}(\Sigma)/\lambda_{\rm min}(\Sigma)$, $\tilde O(\cdot)$ hides logarithmic terms in $\sigma, d, n, 1/\varepsilon, \log(1/\delta), \log(1/\zeta)$ and $\kappa$.
\end{coro}

As a comparison, we also apply the exponential-time  robust linear regression algorithm $\HPTR$ by \citet{liu2022differential} under our setting.

\begin{thm}[{\citep[Theorem~12]{liu2022differential}}]
 	    \label{thm:hptr_main} 
    There exist positive constants $c$ and $C$ such that for any  $((2/11)\alpha,\alpha,\rho_1,\rho_2, \rho_3, \rho_4)$-corrupt good set $S$ with respect to $(w^*,\Sigma\succ 0,\sigma>0)$ satisfying $\alpha< c$, $\rho_1<c$, $\rho_2<c$,   $\rho_3<c$,and $\rho_4^2\leq c\alpha$, ${\rm HPTR}$  achieves 
    $(1/\sigma) \|(\hat\beta-\beta) \|_\Sigma\leq 32 \rho_1  $ with probability $1-\zeta$, if \begin{eqnarray}
         n \;\geq \; C\, \frac{d+\log(1/(\delta\zeta))}{\varepsilon \alpha}  \; .
    \end{eqnarray}
\end{thm}

 We set $\tilde{\alpha}=\alpha^{1-1/k}$, $\rho_1=C_2k(ka)^aK\kappa_2\alpha^{1-1/k}\zeta^{-1/k}$, $\rho_2=C_2K^2{\alpha^{1-1/k}}\log^{2a}(1/{\alpha^{1-1/k}})$,$\rho_3= C_2k^2\kappa_2^2\alpha^{1-2/k}\zeta^{-2/k}$, and $\rho_4=C_2K\alpha^{1-1/k} \log^{a}(1/{\alpha^{1-1/k}})$, we have the following utility gaurentees.
 \begin{coro}
    \label{coro:hptr} 
    Under the hypothesis of Assumption~\ref{asmp:distribution_ht}, 
 there exists a constant $c>0$ such that for any $\alpha\leq c$, $(ka)^aK\kappa_2\alpha^{1-1/k}\zeta^{-1/k}\leq c$, $k^2\kappa_2^2\alpha^{1-2/k}\zeta^{-2/k}\leq c$ and $K^2\alpha^{1-2/k} \log^{2a}(1/{\alpha^{1-1/k}})\leq c$, it is sufficient to have a dataset of size 
 \begin{eqnarray}
	n=O \Big(\frac{d}{\zeta^{2(1-1 / k)} \alpha^{2(1-1 / k)}}+\frac{k^2 \alpha^{2-2 / k} d \log d}{\zeta^{2-4 / k} \kappa_2^2}+\frac{\kappa_2^2 d \log d}{\alpha^{2 / k}}\Big)\;,
\end{eqnarray} 
   such that ${\rm HPTR}$  achieves $(1/\sigma) \| \hat{w}-w^* \|_\Sigma=O(k(ka)^aK\kappa_2\alpha^{1-1/k}\zeta^{-1/k})$ with probability $1-\zeta$. 
 \end{coro}

Note that both of our result in Corollary~\ref{coro:ht_robust} and Corollary~\ref{coro:ht_robust2} are suboptimal compared to the exponential time algorithm $\HPTR$ from Corollary~\ref{coro:hptr}. Suppose $K,a,\kappa_2,k, \zeta=\Theta(1)$, $\HPTR$ achieves $(1/\sigma)\|w^*-\hat{w}\|=\tilde{O}(\alpha^{1-1/k})$ with sample complexities $n=d/(\alpha^{2(1-1/k)})+(d+\log(1/\delta))/(\varepsilon n)$. However, in the analysis in Corollary~\ref{coro:ht_robust}, Algorithm~\ref{alg:main_ht} achieves $(1/\sigma)\|w^*-\hat{w}\|=\tilde{O}(\alpha^{1-2/k})$ with the same sample complexities. In the analysis in Corollary~\ref{coro:ht_robust2}, Algorithm~\ref{alg:main_ht} achieves the same error rate as $\HPTR$ but requires extra $\tilde{O}(d/\alpha^2)$ sample complexities. The suboptimality is caused by the gradient truncation step in our algorithm. From Theorem~\ref{thm:hptr_main}, the final error rate of HPTR only depends on the first resilience $\rho_1$. However in Theorem~\ref{thm:main_ht}, the final error rate of Algorithm~\ref{alg:main_ht} depends on $\hat{\rho}(\alpha)=\max\{\rho_1,\rho_2, \rho_4\sqrt{\rho_2/\alpha}\}$. When the noise is heavy-tailed, the bottleneck is the last term $\rho_4\sqrt{\rho_2/\alpha}\approx\alpha^{1-2/k}$, which is due to the truncation threshold from \Eqref{eq:truncation_bias}. This cannot be tightened by using a smaller truncation threshold. Because we can construct $y_i$, such that there are $\alpha$-fraction of points that are at the threshold level $\theta_t\approx\alpha^{-1/k}$(line~\ref{line:clip1_ht} of Algorithm~\ref{alg:main_ht}). If exponential time complexity is allowed, we could robustly and privately estimate the average of the gradients by directly estimating the $x_iy_i$. However,  the current best efficient algorithm \citep{liu2021robust} for estimating the mean of Gaussian with unknown covariance robustly and privately would require $O(d^{1.5})$ samples.

For a fair comparison, we also rewrite the error rates of Corollary~\ref{coro:ht_robust}, Corollary~\ref{coro:ht_robust2}, Corollary~\ref{coro:hptr} as the same accuracy level $\alpha$ and different corruption level $\alpha_{\rm corrupt}$ respectively.

\begin{coro}
	Under the same hypotheses of Theorem~\ref{thm:main_ht} and under $\alpha_{\rm corrupt}$-corruption model of Assumption~\ref{asmp:corrupt_ht}, if $1.2\alpha_{\rm corrupt}\leq \alpha^{k/(k-2)}$ and $K,a,\kappa_2,k=O(1)$, then 
	$$n=\tilde O(d/(\zeta^{2-2/k}\alpha^{2(k-1)/(k-2)}) + \kappa^{1/2} d/ (\varepsilon \alpha^{(k-1)/(k-2)}))$$ samples are sufficient for Algorithm~\ref{alg:main_ht} to achieve an error rate of $(1/\sigma^2)\|\hat w - w^*\|_\Sigma^2 = \tilde O(\zeta^{-2/k}\alpha^{2})$ with probability $1-\zeta$, where $\kappa:=\lambda_{\rm max}(\Sigma)/\lambda_{\rm min}(\Sigma)$, $\tilde O(\cdot)$ hides logarithmic terms in $\sigma, d, n, 1/\varepsilon, \log(1/\delta), \log(1/\zeta)$ and $\kappa$.
\end{coro}

\begin{coro}
	Under the same hypotheses of Theorem~\ref{thm:main_ht} and under $\alpha_{\rm corrupt}$-corruption model of Assumption~\ref{asmp:corrupt_ht}, if $1.2\alpha_{\rm corrupt}\leq \alpha^{k/(k-1)}$ and $K,a,\kappa_2,k=O(1)$, then 
	$$n=\tilde O(d/(\zeta^{2-2/k}\alpha^2) + \kappa^{1/2} d/ (\varepsilon \alpha^{k/(k-1)})+(d+\log(1/\zeta)/\alpha^{2k/(k-1)}))$$ samples are sufficient for Algorithm~\ref{alg:main_ht} to achieve an error rate of $(1/\sigma^2)\|\hat w - w^*\|_\Sigma^2 = \tilde O(\zeta^{-2/k}\alpha^2)$ with probability $1-\zeta$, where $\kappa:=\lambda_{\rm max}(\Sigma)/\lambda_{\rm min}(\Sigma)$, $\tilde O(\cdot)$ hides logarithmic terms in $\sigma, d, n, 1/\varepsilon, \log(1/\delta), \log(1/\zeta)$ and $\kappa$.
\end{coro}

\begin{coro}[HPTR]
	Under the same hypotheses of Theorem~\ref{thm:main_ht} and under $\alpha_{\rm corrupt}$-corruption model of Assumption~\ref{asmp:corrupt_ht}, if $\alpha_{\rm corrupt}\leq \alpha^{k/(k-1)}$ and $\alpha^{(k-2)/(k-1)}\leq c$ and $K,a,\kappa_2,k=O(1)$, then 
	$$n=\tilde O(\frac{d}{\zeta^{2-2/k}\alpha^{2}}+ \frac{d+\log(1/(\delta\zeta))}{\varepsilon\alpha^{k/k-1}})$$ samples are sufficient for HPTR to achieve an error rate of $(1/\sigma^2)\|\hat w - w^*\|_\Sigma^2 = \tilde O(\zeta^{-2/k}\alpha^2)$ with probability $1-\zeta$, $\tilde O(\cdot)$ hides logarithmic terms in $\sigma, d, n, 1/\varepsilon, \log(1/\delta), \log(1/\zeta)$ and $\kappa$.
\end{coro}
%
%


\subsection{Proof of Theorem~\ref{thm:main_ht}}
\label{app:proof_main_ht}
\begin{proof}
	The proof follows similarly as the proof of Theorem~\ref{thm:main}.  We only highlight the difference in the proof.

Let $S_{\rm good}$ be the uncorrupted dataset for $S_3$ and $S_{\rm bad}$ be the corrupted data points in $S_3$. 
Let $G$ denote the clean data that satisfies resilience conditions. We know $|G|\geq (1-1.2\alpha_{\rm corrupt})n\geq (1-\alpha)n$.

Let $\lambda_{\rm max}=\|\Sigma\|_2$. 
Define $\hat{\Sigma}:=(1/n) \sum_{i\in G}x_ix_i^\top$, $\hat{B}:=\mathbf{I}_d-\eta\hat{\Sigma}$. Lemma~\ref{lemma:cov_concentration_tail} implies that if $n=O(K^2d\log(d/\zeta)\log^{2a}(n/\zeta))$, then
\begin{align}
	0.9\Sigma\preceq \hat{\Sigma}\preceq 1.1\Sigma\;. \label{eq:cov_asmp_ht}
\end{align}
We pick step size $\eta$ such that $\eta\leq 1/(1.1\lambda_{\max})$ to ensure that  $\eta\leq 1/\|\hat{\Sigma}\|_2$. 
Since the covariates $\{x_i\}_{i\in S}$ are not corrupted, from Lemma~\ref{lemma:norm_a_tail}, we know with probability $1-\zeta$, for all $i\in S_3$, 
\begin{align}
	\|x_i\|^2\leq K^2\Tr(\Sigma)\log^{2a}(n/\zeta)\label{eq:norm_asmp_ht}\;.
\end{align}

The rest of the proof is under \Eqref{eq:cov_asmp_ht}, \Eqref{eq:norm_asmp_ht} and the resilience conditions.

Let $\phi_t=(\sqrt{2\log(1.25/\delta_0)}\Theta\theta_t)/(\varepsilon_0 n)$.   
Define $g_i^{(t)}:=x_i(x_i^\top w_t-y_i)$. For $i\in S_{\rm good}$, we know $y_i=x_i^\top w^*+z_i$. Let $\tilde{g}_i^{(t)}={\rm clip}_{\Theta}(x_i){\rm clip}_{\theta_t}(x_i^\top w_t-y_i)$. Note that under \Eqref{eq:norm_asmp_ht}, ${\rm clip}_\Theta(x_i)=x_i$ for all $i\in S_3$. 

From Algorithm~\ref{alg:main_ht}, we can write one-step update rule as follows:
\begin{align}
	&w_{t+1}-w^* \nonumber \\
	=& w_{t}-\eta\left(\frac{1}{n}\sum_{i\in S}\tilde{g}_i^{(t)}+\phi_t\nu_t\right) -w^*\nonumber\\
	=&\left(\mathbf{I}-\frac{\eta}{n}\sum_{i\in G}x_ix_i^\top\right)(w_{t}-w^*)+\frac{\eta}{n}\sum_{i\in G}x_iz_i+
	\frac{\eta}{n}\sum_{i\in G}(g_i^{(t)}-\tilde{g}_i^{(t)}) - \eta\phi_t\nu_t
	-\frac{\eta}{n}\sum_{i\in S_3\setminus G \cup E_t}\tilde{g}_i^{(t)} 
	\label{eq:onestep_ht}
\end{align}
Let  $E_t:=\{i\in G: \theta_t\leq |x_i^\top w_t-y_i|\}$ be the set of clipped clean data points such that $\sum_{i\in G}(g_i^{(t)}-\tilde{g}_i^{(t)})=\sum_{i\in E_t }(g_i^{(t)}-\tilde{g}_i^{(t)})$. We define $\hat{v} :=(1/n)\sum_{i\in G}x_iz_i $,   $u_t^{(1)}:=(1/n)\sum_{i\in E_t} x_ix_i^\top (w_t-w^*)$, $u_t^{(2)} := (1/n)\sum_{i\in E_t} -x_iz_i$, and $u_t^{(3)} := (1/n) \sum_{i\in S_3\setminus G\cup E_t}\tilde{g}_i^{(t)}$. 

We can further write the update rule as:
\begin{align}
    w_{t+1}-w^*
    =&\hat{B}(w_{t}-w^*)+\eta \hat{v}+\eta u_{t-1}^{(1)}+\eta u_{t-1}^{(2)}-\eta \phi_t\nu_t-\eta u_{t-1}^{(3)}\;. \label{eq:update_ht}
\end{align}
Since $G\subset S_{\rm good}$ and $|G|\geq(1-\alpha)n$, using the resilience property in \Eqref{def:res1}, we know
\begin{align}
    \|\Sigma^{-1/2}\hat{v}\|&=|G|\max_{\|v\|=1}\Sigma^{-1/2}\ip{v}{\frac{1}{|G|}\sum_{i\in G}x_iz_i} \nonumber \\
    &\leq (1-\alpha)\rho_1\sigma\\
    &\leq  \rho_1\sigma\;.\label{eq:hat_v_ht}
\end{align}

Let $\alpha_2=|E_t|/n$. Following the proof of Lemma~\ref{lemma:clipping_fraction}, we can show following lemma. 
\begin{lemma}
\label{lemma:clipping_fraction_ht}
Under Assumptions~\ref{asmp:distribution_ht}, if $
	\theta_t  \geq \sqrt{\max\{8\rho_2/\alpha, 8\rho_3/\alpha\}+1} \cdot \left(\|w^*-w_t\|_\Sigma+\sigma\right)$, 
then  
$$
	\left|\left\{i\in G: \left|w_t^\top x_i-y_i\right|\geq \theta_t\right\}\right|  \leq  \alpha n$$, for all $t\in [T]$.
\end{lemma} 

Similar as Theorem~\ref{thm:distance}, we have following theorem.
\begin{thm}
\label{thm:distance_ht}
	 Algorithm~\ref{alg:distance} is $(\varepsilon_0, \delta_0)$-DP. 
	 For an $(\alpha_{\rm corrupt}, \bar{\alpha}, \rho_1, \rho_2,\rho_3,\rho_4)$-corrupted good dataset $S_2$ and an upper bound $\bar\alpha$ on $\alpha_{\rm corrupt}$ that satisfy  Assumption~\ref{asmp:distribution_ht} and $\rho_1 +\rho_2+\rho_3\leq 1/4$, for any $\zeta\in (0,1)$, if  
    \begin{align}
		n\;=\;O\left(\frac{\log(1/\zeta)\log(1/(\delta_0\zeta))}{ \bar \alpha \varepsilon_0}
		\right)\;, \label{eq:distance_ht}
	\end{align} 
	with a large enough constant then, with probability $1-\zeta$,   Algorithm~\ref{alg:distance} returns $\ell$ such that 
$
	\frac{1}{4}(\|w_t-w^*\|_{\Sigma}^2+\sigma^2) \;\leq\; \ell\; 
	\leq\; 
	 4(\|w_t-w^*\|_{\Sigma}^2+\sigma^2)
	 $.
\end{thm}

This means $\alpha_2\leq \alpha$, and we have
\begin{align*}
    \|\Sigma^{-1/2} u_t^{(1)}\| &=  \|\Sigma^{-1/2} \frac{1}{n}\sum_{i\in E_t} x_ix_i^\top (w_t-w^*)\|   \;.
\end{align*} 
From Corollary~\ref{coro:res}, we know 
\begin{align*}
    &\left|\|\Sigma^{-1/2} \frac{1}{|E_t|}\sum_{i\in E_t} x_ix_i^\top (w_t-w^*)\|- \|w_t-w^*\|_{\Sigma}\right|\\
    = &\left|\max_{u:\|u\|=1} \frac{1}{|E_t|}\sum_{i\in E_t} u^\top \Sigma^{-1/2} x_ix_i^\top (w_t-w^*)\|- \max_{v:\|v\|=1}v^\top\Sigma^{1/2}(w_t-w^*)\right|\\
    \leq &\max_{u:\|u\|=1}\left| \frac{1}{|E_t|}\sum_{i\in E_t} u^\top  \Sigma^{-1/2}x_ix_i^\top\Sigma^{-1/2} \Sigma^{1/2} (w_t-w^*)\|- u^\top\Sigma^{1/2}(w_t-w^*)\right|\\ 
    \leq &\max_{u:\|u\|=1}\left| \frac{1}{|E_t|}\sum_{i\in E_t} u^\top  \left(\Sigma^{-1/2}x_ix_i^\top\Sigma^{-1/2}-\mathbf{I}_d\right) \Sigma^{1/2} (w_t-w^*)\|\right|\\
    =&\left\| \frac{1}{|E_t|}\sum_{i\in E_t}   \left(\Sigma^{-1/2}x_ix_i^\top\Sigma^{-1/2}-\mathbf{I}_d\right) \Sigma^{1/2} (w_t-w^*)\right\|\\
    \leq &\left\| \frac{1}{|E_t|}\sum_{i\in E_t}   \left(\Sigma^{-1/2}x_ix_i^\top\Sigma^{-1/2}-\mathbf{I}_d\right)\right\|\cdot \left\| \Sigma^{1/2} (w_t-w^*)\right\|\\
    \leq& \frac{2-\alpha_2}{\alpha_2}\rho_2\left\| w_t-w^*\right\|_\Sigma\;.
\end{align*}
This implies that 
\begin{align}
   \|\Sigma^{-1/2} u_t^{(1)}\|&\;\leq \; \|\Sigma^{-1/2} \frac{1}{n}\sum_{i\in E} x_ix_i^\top (w_t-w^*)\| \nonumber\\
   &\leq \left(\alpha_2+2\rho_2\right)\left\| w_t-w^*\right\|_\Sigma \nonumber\\
   &\leq 3\rho_2\left\| w_t-w^*\right\|_\Sigma\;,
   \label{eq:ut1_ht}
\end{align}
where the last inequality follows from  the fact that $\alpha_2\leq \alpha$ and our assumption that $\alpha\leq \rho_2$ from Assumption~\ref{asmp:corrupt_ht}. 
Similarly, we use resilience property in \Eqref{def:res1} instead of \Eqref{def:res2}, we can show that 
\begin{align}
    \|\Sigma^{-1/2}u_t^{(2)}\| \leq 3\rho_3 \sigma\;. \label{eq:ut2_ht}
\end{align}

Next, we consider $u_t^{(3)}$. Since $|S_3\setminus G|\leq 1.2\alpha_{\rm corrupt}n$ and $|E_t|\leq \alpha n$,  using \Eqref{def:res4} and Corollary~\ref{coro:res}, we have
\begin{align}
    \|\Sigma^{-1/2}u_t^{(3)}\|&=\max_{v:\|v\|=1}\frac{1}{n}\sum_{i\in S_{\rm bad} \cup E_t}v^\top\Sigma^{-1/2}x_i{\rm clip}_{\theta_t}(x_i^\top w_t-y_i) \nonumber \\
    &\leq 2\rho_4\theta_t \nonumber \\
    &\leq 2\rho_4\sqrt{8\max\{\rho_2/\alpha,\rho_3/\alpha\}+1}\cdot (\|w_{t}-w^*\|_\Sigma +\sigma)\;.\label{eq:truncation_bias}
\end{align}

The analysis of convergence follows similarly as in Step~3 and Step~4 of the proof of Theorem~\ref{thm:main} except we set $\hat{\rho}(\alpha)=\max\{\rho_1, 3\rho_2, 2\rho_4\sqrt{8\max\{\rho_2/\alpha,\rho_3/\alpha\}+1}\}$. 

The second term in \Eqref{eq:main_n_hyper} ensures the added Gaussian noise is small enough such that $\phi_t^2\|v_t\|^2\leq \hat\rho(\alpha)^2(\E[\|w_{t}-w^*\|_{\Sigma}^2] +  \sigma^2)$, which is similar as in \Eqref{eq:noise_ineq_main}

\end{proof}

\subsection{Proof of Lemma~\ref{lemma:ht_conditions}}
\label{app:proof_ht_conditions}
 \begin{proof} 	
 
  For any $x$ that is $(K,a)$-sub-Weibull from Definition~\ref{def:a_tail}, \Eqref{eq:tail_weibull} implies that for any $k\geq 1$,
  \begin{align}
  	\E[|\ip{v}{x}|^k]&=\int_{0}^\infty \prob(|\ip{v}{x}|\geq t^{1/k})dt\\
  	&\leq \int_{0}^\infty 2\exp\left(-\frac{t^{\frac{1}{ka}}}{(K^2\E[\ip{v}{x}^2])^{\frac{1}{2a}}}\right)dt\\
  	&=2K^{k}(\E[\ip{v}{x}^2])^{k/2}ka\int_{0}^\infty e^{-u}u^{ka-1}du\\
  	&=2K^{k}(\E[\ip{v}{x}^2])^{k/2}\Gamma(ka+1)\\
&\leq 2K^{k}(\E[\ip{v}{x}^2])^{k/2}(ka)^{ka}  	  
\end{align}
 	
 	This implies that $x_i$ is also $((ka)^aK, k )$-hypercontractive. Since $x_i$ and $z_i$ are independent, we have 	 
	\begin{align}
 	 	\mathbb{E}\left[\left|\left\langle v, \sigma^{-1} \Sigma^{-1 / 2} x_i z_i\right\rangle\right|^k\right]=\mathbb{E}\left[\left|\left\langle v, \Sigma^{-1 / 2} x_i\right\rangle\right|^k\right] \mathbb{E}\left[\left|\sigma^{-1} z_i\right|^k\right] \leq 2(ka)^{ka}K^k \kappa_2^k\;.
 	 \end{align} 
 	 This means $x_iz_i$ is also $((ka)^aK\kappa_2, k )$-hypercontractive. From \citet[Lemma~G.10]{zhu2019generalized}, we know with probability $1-\zeta$, there exists $S_1\subset S_{\rm good}$ with $|S_1|\geq(1-0.1\alpha)|S_{\rm good}|$, such that for any $T\subset S_1$ with $|T|\geq (1-\alpha)|S_1|$, we have
 	 \begin{align}
 	 	\Big| \frac{1}{|T|}\sum_{(x_i, y_i)\in S} \ip{v}{\sigma^{-1}\Sigma^{-1/2}x_i(y_i-x_i^\top w^*)}  \Big|\leq C_2k(ka)^aK\kappa_2\alpha^{1-1/k}\zeta^{-1/k}\;.
 	 \end{align}
 	 
 	 Similarly, there exists $S_2\subset S_{\rm good}$ with $|S_2|\geq(1-0.1\alpha)|S_{\rm good}|$, such that for any $T\subset S_2$ with $|T|\geq (1-\alpha)|S_2|$, we have
 	 \begin{align}
 	 	\Big| \frac{1}{|T|}\sum_{(x_i, y_i)\in T} (\sigma^{-1}(y_i-x_i^\top w^*))^2-1  \Big|\leq C_2k^2\kappa_2^2\alpha^{1-2/k}\zeta^{-2/k}\;.
 	 \end{align}
	 
	 From Lemma~\ref{lemma:subweibull_res_conditions}, for any $T\subset S_{\rm good}$ with $|T|\geq (1-\tilde\alpha)|S_{\rm good}|$, we have
 	 \begin{align}
 	 	\Big| \frac{1}{|T|}\sum_{(x_i, y_i)\in T} \ip{v}{\Sigma^{-1/2}x_i}^2-1  \Big|\leq C_2K\tilde\alpha \log^{2a}(1/\tilde{\alpha})\;.
 	 \end{align}
 	 
 	 and
	 \begin{align}
		\Big| \frac{1}{|T|}\sum_{(x_i, y_i)\in T} \ip{v}{\Sigma^{-1/2}x_i}  \Big|\leq C_2K\tilde\alpha \log^{a}(1/\tilde{\alpha})\;.
	 \end{align} 
 	 
 	 Set $S=S_1\cap S_2$, we know $|S|\geq (1-0.2\alpha)|S_{\rm good}|$ and $S$ is \\
	  $(0.2\alpha, \alpha, C_2k(ka)^aK\kappa_2\alpha^{1-1/k}\zeta^{-1/k},C_2K^2\tilde{\alpha}\log^{2a}(1/\tilde{\alpha}),C_2k^2\kappa_2^2\alpha^{1-2/k}\zeta^{-2/k}, C_2K\tilde\alpha \log^{a}(1/\tilde{\alpha}))$-corrupt good with respect to $(w^*,\Sigma,\sigma)$.  This completes the proof.
 	\end{proof}
\end{document}

%% file: math_commands.tex
\usepackage{amsmath,amsfonts,bm}

\def\eqref#1{Eq.~(\ref{#1})}
\def\Eqref#1{Eq.~(\ref{#1})}

\def\1{\bm{1}}

\def\eps{{\epsilon}}

\DeclareMathAlphabet{\mathsfit}{\encodingdefault}{\sfdefault}{m}{sl}
\SetMathAlphabet{\mathsfit}{bold}{\encodingdefault}{\sfdefault}{bx}{n}

\newcommand{\E}{\mathbb{E}}

\DeclareMathOperator*{\argmin}{arg\,min}

\DeclareMathOperator{\Tr}{Tr}

\newcommand{\dprobgd}{\textsc{DP-RobGD}}
\newcommand{\dprobgdstar}{\textsc{DP-RobGD*}}
\newcommand{\ssp}{\textsc{DP-SSP}}
\newcommand{\dpsgd}{\textsc{DP-AMBSSGD}}
\newcommand{\sgd}{\textsc{SGD}}
\newcommand{\ols}{\textsc{OLS}}